\documentclass[11pt]{article} 

\usepackage{graphicx} 
\usepackage{subfigure} 
\usepackage{amssymb}
\usepackage{amsmath}
\usepackage{comment}

\usepackage{algorithm}
\usepackage{algorithmic}

\usepackage{hyperref}

\def\bi{\begin{itemize}}
\def\ei{\end{itemize}}

\def\N{{\mathcal{N}}}
\def\beas{\begin{eqnarray*}}
\def\eeas{\end{eqnarray*}}
\def\bea{\begin{eqnarray}}
\def\eea{\end{eqnarray}}
\def\beq{\begin{equation}}
\def\eeq{\end{equation}}
\def\bit{\begin{itemize}}
\def\eit{\end{itemize}}
\def\ben{\begin{enumerate}}
\def\een{\end{enumerate}}
\def\BA{\begin{array}}
\def\EA{\end{array}}

\def\ve{\mathbf{\varepsilon}}

\input{macros}

\newcommand{\prox}{ { \rm prox}}

\newcommand{\rb}{\mathbb{R}}
\newcommand{\BlackBox}{\rule{1.5ex}{1.5ex}}  

\newenvironment{proof}{\par\noindent{\bf Proof\ }}{\hfill\BlackBox\\[2mm]}
\newtheorem{lemma}{Lemma}

\newtheorem{proposition}{Proposition}

\newcommand{\BEAS}{\begin{eqnarray*}}
\newcommand{\EEAS}{\end{eqnarray*}}
\newcommand{\BEA}{\begin{eqnarray}}
\newcommand{\EEA}{\end{eqnarray}}
\newcommand{\BEQ}{\begin{equation}}
\newcommand{\EEQ}{\end{equation}}
\newcommand{\BIT}{\begin{itemize}}
\newcommand{\EIT}{\end{itemize}}
\newcommand{\BNUM}{\begin{enumerate}}
\newcommand{\ENUM}{\end{enumerate}}

\title{Convergence Rates of Inexact Proximal-Gradient Methods for Convex Optimization}

\author{
Mark Schmidt\\
\texttt{mark.schmidt@inria.fr}
\and
Nicolas Le Roux\\
\texttt{nicolas@le-roux.name}
\and
Francis Bach\\
\texttt{francis.bach@ens.fr}
\and
\\
INRIA - SIERRA Project - Team\\
Laboratoire d'Informatique de l'\'Ecole Normale Sup\'erieure\\
Paris, France
}

\date{\today}

\begin{document}

\maketitle

\begin{abstract}
We consider the problem of optimizing the sum of a smooth convex function and a non-smooth convex function using proximal-gradient methods, where an error is present in the calculation of the gradient of the smooth term or in the proximity operator with respect to the non-smooth term. 
We show that both the basic proximal-gradient method and the accelerated proximal-gradient method achieve the same convergence rate as in the error-free case, provided that the errors decrease at appropriate rates.
Using these rates, we perform as well as or better than a carefully chosen fixed error level on a set of structured sparsity problems.
\end{abstract}

\section{Introduction}

In recent years the importance of taking advantage of the structure of convex optimization problems has become a topic of intense research in the machine learning community. This is particularly true of techniques for non-smooth optimization, where taking advantage of the structure of non-smooth terms seems to be crucial to obtaining good performance.  Proximal-gradient methods and \emph{accelerated} proximal-gradient methods~\cite{beck2009fast,nesterov2007gradient} are among the most important methods for taking advantage of the structure of many of the non-smooth optimization problems that arise in practice.  In particular, these methods address composite optimization problems of the form
\begin{equation}
\label{eq:1}
\minimize{x\in\Real^d}\quad f(x) \defd g(x) + h(x),
\end{equation}
where $g$ and $h$ are convex functions but only $g$ is smooth.  One of the
most well-studied instances of this type of problem is $\ell_1$-regularized
least squares~\cite{tibshirani1996regression,chen2001atomic},
\[
\minimize{x\in\Real^d}\; \frac{1}{2}\norm{Ax-b}^2 + \lambda\norm{x}_1,
\]
where we use $\|\cdot\|$ to denote the standard $\ell_2$-norm.

Proximal-gradient methods are an appealing approach for solving these types of non-smooth optimization
problems because of their fast theoretical convergence rates and strong practical performance.
While classical subgradient methods only achieve an error level on the objective function of $O(1/\sqrt{k})$ after $k$ iterations, proximal-gradient methods have an error of $O(1/k)$ while \emph{accelerated} proximal-gradient methods futher reduce this to $O(1/k^2)$~\cite{beck2009fast,nesterov2007gradient}.  That
is, accelerated proximal-gradient methods for \emph{non-smooth} convex optimization
achieve the same optimal convergence rate that accelerated gradient methods
achieve for \emph{smooth} optimization.

Each iteration of a proximal-gradient method
requires the calculation of the proximity operator,
\begin{equation}
\label{eq:prox}
\prox_L(y) = \argmin_{x\in\Real^d}\; \frac{L}{2}\norm{x-y}^2 +
h(x),
\end{equation}
where $L$ is the Lipschitz constant of the gradient of $g$.
We can efficiently compute an analytic solution to this problem for several
notable choices of $h$, including the case of $\ell_1$-regularization and disjoint
group $\ell_1$-regularization~\cite{wright2009sparse,back_convex_sparsity_11}.  
However, in many scenarios the proximity operator may not have an analytic solution,
or it may be very expensive to compute this solution exactly.
 This includes important problems such as
total-variation regularization and its generalizations like the graph-guided fused-LASSO~\cite{fadili-tip-10-tv,chen2010graph},
nuclear-norm regularization and other regularizers on the singular values of
matrices~\cite{cai2008singular,ma2009fixed}, and different formulations of
overlapping group $\ell_1$-regularization with general
groups~\cite{jacob2009group,jenatton2010proximal}.  
Despite the difficulty in computing the exact proximity operator for these
regularizers, efficient methods have been developed to compute \emph{approximate}
proximity operators in all of these cases; accelerated projected gradient and Newton-like methods that work with a smooth dual problem have been used to compute approximate proximity operators in the context of total-variation
regularization~\cite{fadili-tip-10-tv,barbero2011Newton}, Krylov subspace methods and low-rank representations have been used to compute approximate proximity operators in the context of nuclear-norm
regularization~\cite{cai2008singular,ma2009fixed}, and variants of Dykstra's algorithm (and related dual methods) have been used to compute approximate proximity operators in the context of overlapping group
$\ell_1$-regularization~\cite{jenatton2010proximal,liu2010fast,schmidt2010convex}.  

It is known that proximal-gradient methods that use an approximate proximity operator converge under only weak assumptions~\cite{patriksson1999nonlinear,combettes2004solving}; we briefly review this and other related work in the next section. 
However, despite the many recent works showing impressive empirical performance of (accelerated) proximal-gradient methods that use an approximate proximity operator~\cite{fadili-tip-10-tv,cai2008singular,ma2009fixed,barbero2011Newton,liu2010fast,schmidt2010convex}, 
up until recently there was no theoretical analysis on how the error in the calculation of the proximity operator affects the convergence \emph{rate} of proximal-gradient methods.
In this work, we show in several contexts that, provided the error in the proximity operator calculation is controlled in an appropriate way, inexact proximal-gradient strategies achieve the \emph{same} convergence rates as the corresponding exact methods. In particular, in Section~\ref{sec:rates} we first consider convex objectives and analyze the inexact proximal-gradient (Proposition~\ref{prop:convex}) and accelerated proximal-gradient (Proposition~\ref{prop:accelerated_convex}) methods. 
We then analyze these two algorithms for strongly convex objectives (Proposition~\ref{prop:strongly_convex} and Proposition~\ref{prop:accelerated_strongly_convex}). 
Note that, in these analyses, we also consider the possibility that there is an error in the calculation of the gradient of $g$.  We then present an experimental comparison of various inexact proximal-gradient strategies in the context of solving a structured sparsity problem (Section \ref{sec:experiments}).

\section{Related Work}

The algorithm we shall focus on in this paper is the proximal-gradient
method
\beq
x_{k} = \prox_L\left[y_{k-1} - (1/L) (g'(y_{k-1}) + e_{k})\right] \; ,
\label{eq:2}
\eeq
where $e_k$ is the error in the calculation of the gradient and the proximity problem~(\ref{eq:prox}) is solved inexactly so that $x_k$ has an error of $\ve_k$ in terms of the proximal objective function~\eqref{eq:prox}.
In the basic proximal-gradient method we choose $y_k = x_k$, while in the accelerated proximal-gradient method we choose
\[
y_k = x_k + \beta_k(x_k - x_{k-1}),
\]
where the sequence $\{\beta_k\}$ is chosen to accelerate the convergence rate.

There is a substantial amount of work on methods that use an exact proximity operator but have an error in
the gradient calculation, corresponding to the special case where $\ve_k = 0$ but $e_k$ is non-zero.
For example, when the $e_k$ are independent, zero-mean, and finite-variance random variables, then proximal-gradient
methods achieve the (optimal) error level of $O(1/\sqrt{k})$~\cite{duchi2009efficient,langford2009sparse}. This is different than the scenario we analyze in this paper since we do \emph{not} assume unbiased nor independent errors but instead consider a sequence of errors converging to 0.  This leads to faster convergence rates and makes our analysis applicable to the case of deterministic and even adversarial errors. 

Several authors have recently analyzed the case of a fixed deterministic error in the
gradient, and shown that accelerated gradient methods
achieve the optimal convergence rate up to some accuracy that depends on the
fixed error level~\cite{d2005smooth,baes2009estimate,devolder2010first}, while the earlier work of~\cite{nedic2000convergence} analyzes the gradient method in the context of a fixed error level.  This contrasts with our analysis where by allowing the error to change at every iteration we can achieve convergence to the optimal solution. Also, we can tolerate a large error in early iterations when we are far from the solution, which may lead to substantial computational gains. Other authors have
analyzed the convergence rate of the gradient and projected-gradient methods
with a decreasing sequence of
errors~\cite{luo1993error,friedlander2011hybrid} but this analysis 
does not consider the important class of accelerated gradient methods.  In contrast, the analysis of~\cite{devolder2010first} allows a decreasing sequence of errors 
(though convergence rates in this context are not explicitly mentioned) and considers the accelerated projected-gradient method.  However, the authors of this work only consider the case of an exact projection step and they assume the availability of an oracle that yields global lower and upper bounds on the function.  This non-intuitive oracle leads to a novel analysis of smoothing methods, but also to slower convergence rates than proximal-gradient methods.  The analysis of~\cite{baes2009estimate} considers errors in both the gradient and projection operators for accelerated projected-gradient methods but requires that the domain of the function is compact. None of these works consider proximal-gradient methods.

In the context of \emph{proximal-point} algorithms, there is a substantial literature on using inexact proximity operators with a decreasing sequence of errors, dating back to the seminal work of Rockafellar~\cite{rockafellar1976monotone}.  Accelerated proximal-point methods with a decreasing sequence of errors have also been examined, beginning with~\cite{guler1992new}.  However, unlike proximal-gradient methods where the proximity operator is only computed with respect to the non-smooth function $h$, proximal-point methods require the calculation of the proximity operator with respect to the full objective function.  In the context of composite optimization problems of the form~\eqref{eq:1}, this requires the calculation of the proximity operator with respect to $g+h$.  Since it ignores the structure of the problem, this proximity operator may be as difficult to compute (even approximately) as the minimizer of the original problem.

Convergence of inexact proximal-gradient methods can be established with only weak assumptions on the method used to approximately solve~\eqref{eq:prox}.  For example, we can establish that inexact proximal-gradient methods converge under some closedness assumptions on the mapping induced by the approximate proximity operator, and the assumption that the algorithm used to compute the inexact proximity operator achieves sufficient descent on problem~\eqref{eq:prox} compared to the previous iteration $x_{k-1}$~\cite{patriksson1999nonlinear}.  Convergence of inexact proximal-gradient methods can also be established under the assumption that the norms of the errors are summable~\cite{combettes2004solving}.  However, these prior works did not consider the \emph{rate} of convergence of inexact proximal-gradient methods, nor did they consider accelerated proximal-gradient methods.  
Indeed, the authors of~\cite{fadili-tip-10-tv} chose to use the non-accelerated variant of the proximal-gradient algorithm since even convergence of the accelerated proximal-gradient method had not been established under an inexact proximity operator.

While preparing the final version of this work,~\cite{villa2011accelerated} independently gave an analysis of the accelerated proximal-gradient method with an inexact proximity operator and a decreasing sequence of errors (assuming an exact gradient).  Further, their analysis leads to a weaker dependence on the errors than in our Proposition~\ref{prop:accelerated_convex}.  However, while we only assume that the proximal problem can be solved up to a certain accuracy, they make the much stronger assumption that the inexact proximity operator yields an $\ve_k$-subdifferential of $h$~\cite[Definition 2.1]{villa2011accelerated}.  Our analysis can be modified to give an improved dependence on the errors under this stronger assumption. In particular, the terms in $\sqrt{\ve_i}$ disappear from the expressions of $A_k$, $\widetilde{A}_k$ and $\widehat{A}_k$. In the case of Propositions~\ref{prop:convex} and~\ref{prop:accelerated_convex}, this leads to the optimal convergence rate with a slower decay of $\ve_i$. More details may be found after Lemma~\ref{lemma:subdifferential} in the Appendix.  More recently,~\cite{jianginexact} gave an alternative analysis of an accelerated proximal-gradient method with an inexact proximity operator and a decreasing sequence of errors (assuming an exact gradient), but under a non-intuitive assumption on the relationship between the approximate solution of the proximal problem and the $\ve_k$-subdifferential of $h$.

\section{Notation and Assumptions}

In this work, we assume that the smooth function $g$ in~\eqref{eq:1} is convex and differentiable, and that its gradient $g'$ is Lipschitz-continuous with constant $L$, meaning that for all $x$ and $y$ in $\Real^d$ we have
\[
\norm{g'(x)-g'(y)} \leqslant L\norm{x-y} \; .
\]
This is a standard assumption in differentiable optimization, see~\cite[\S2.1.1]{nesterov2004introductory}.  If $g$ is twice-differentiable, this corresponds to the assumption that the eigenvalues of its Hessian are bounded above by $L$.
In Propositions~\ref{prop:strongly_convex} and~\ref{prop:accelerated_strongly_convex} only, we will also assume that $g$ is $\mu$-strongly convex (see~\cite[\S2.1.3]{nesterov2004introductory}), meaning that for all $x$ and $y$ in $\Real^d$ we have
\[
g(y) \geqslant g(x) + \langle g'(x),y-x \rangle + \frac{\mu}{2}||y-x||^2.
\]
However, apart from Propositions~\ref{prop:strongly_convex} and~\ref{prop:accelerated_strongly_convex}, we only assume that this holds with $\mu=0$, which is equivalent to convexity of $g$.

In contrast to these assumptions on $g$, we will only assume that $h$ in~\eqref{eq:1} is a lower semi-continuous proper convex function (see~\cite[\S1.2]{bertsekas2009convex}), but will not assume that $h$ is differentiable or Lipschitz-continuous.  This allows $h$ to be any real-valued convex function, but also allows for the possibility that $h$ is an extended real-valued convex function.  For example, $h$ could be the indicator function of a convex set, and in this case the proximity operator becomes the projection operator.

We will use $x_k$ to denote the parameter vector at iteration $k$, and $x^*$ to denote a minimizer of $f$.  
We assume that such an $x^*$ exists, but do not assume that it is unique.
 We use $e_k$
to denote the error in the calculation of the gradient at iteration $k$, and we use $\ve_k$ to denote the error in the proximal objective function achieved by $x_{k}$, meaning
that
\begin{equation}
\label{eq:3}
\frac{L}{2}\norm{x_{k}-y}^2 + h(x_{k}) \leqslant \ve_k + \min_{x\in\Real^d}\; \bigg\{\frac{L}{2}\norm{x-y}^2 +
h(x)\bigg\},
\end{equation}
where $y = y_{k-1} - (1/L)(g'(y_{k-1}) + e_k))$.  
Note that the proximal optimization problem~\eqref{eq:prox} is strongly convex and in practice we are often able to obtain such bounds via a duality gap (e.g., see~\cite{jenatton2010proximal} for the case of overlapping group $\ell_1$-regularization).

\section{Convergence Rates of Inexact Proximal-Gradient Methods}
\label{sec:rates}
In this section we present the analysis of the convergence rates of inexact proximal-gradient methods as a function of the sequences of solution accuracies to the proximal problems $\{\ve_k\}$, and the sequences of magnitudes of the errors in the gradient calculations $\{\|e_k\|\}$. We shall use (\textbf{H}) to denote the set of four assumptions which will be made for each proposition:
\bit
\item $g$ is convex and has $L$-Lipschitz-continuous gradient;
\item $h$ is a lower semi-continuous proper convex function;
\item The function $f = g +h$ attains its minimum at a certain $x^\ast \in \rb^n$;
\item $x_k$ is an $\ve_k$-optimal solution to the proximal problem~\eqref{eq:prox} in the sense of~\eqref{eq:3}.
\eit

We first consider the basic proximal-gradient method in the convex case:
\begin{proposition}[Basic proximal-gradient method - Convexity]
\label{prop:convex}
Assume (\textbf{H}) and that we iterate recursion~\eqref{eq:2} with $y_k = x_k$. Then, for all $k \geqslant 1$, we have 
   \BEQ
   f\left(\frac{1}{k}\sum_{i=1}^k x_i \right) - f(x^\ast) \leqslant
   \frac{L}{2k}
\left( \| x_0 -x^\ast\| + 2A_k  + \sqrt{2B_k}  \right) ^2 \; ,
   \EEQ
with
   \[
   A_k = \sum_{i=1}^k \left(\frac{\|  e_i \|}{L} +   \sqrt{\frac{2\varepsilon_i}{L}}\right) \;, \quad B_k = \sum_{i=1}^k \frac{\varepsilon_i}{L} \; .
\]

\end{proposition}

The proof may be found in the Appendix.
Note that while we have stated the proposition in terms of the function value achieved by the average of the iterates, it trivially also holds for the iteration that achieves the lowest function value. 
This result implies that the well-known $O(1/k)$ convergence rate for the gradient method without errors \emph{still holds} when both $\{\|e_k\|\}$ and $\{\sqrt{\ve_k}\}$ are summable. A sufficient condition to achieve this is for $\|e_k\|$ and $\sqrt{\ve_k}$ to decrease as $O(1/k^{1+\delta})$ for any $\delta > 0$.  Note that a faster convergence of these two errors will not improve the convergence rate but will yield a better constant factor.

It is interesting to consider what happens if $\{\|e_k\|\}$ or $\{\sqrt{\ve_k}\}$ is not summable. For instance, if $\|e_k\|$ and $\sqrt{\ve_k}$ decrease as $O(1/k)$, then $A_k$ grows as $O(\log k)$ (note that $B_k$ is always smaller than $A_k$) and the convergence of the function values is in $O\left(\frac{\log^2 k}{k}\right)$.  Finally, a necessary condition to obtain convergence is that the partial sums $A_k$ and $B_k$ need to be in $o(\sqrt{k})$.

We now turn to the case of an \emph{accelerated} proximal-gradient method.  We focus on a basic variant of the algorithm where $\beta_k$ is set to~$(k-1)/(k+2)$~\cite[Eq.~(19) and (27)]{tseng2008accelerated}:
\begin{proposition}[Accelerated proximal-gradient method - Convexity]
\label{prop:accelerated_convex}
Assume~(\textbf{H}) and that we iterate recursion~\eqref{eq:2} with $y_k = x_k + \frac{k-1}{k+2}(x_k - x_{k-1})$.
Then, for all $k \geqslant 1$, we have
\BEQ
\label{eq:convex_accelerated_bound}
f(x_k)-f(x^*) \leqslant
 \frac{2L}{(k+1)^2}
\left(
\| x_0 -x^\ast\| + 2\widetilde{A}_k + \sqrt{2\widetilde{B}_k}  \right) ^2 ,
\EEQ
with
\[
\widetilde{A}_k = \sum_{i=1}^k i \left(\frac{\|  e_i \|}{L} +   \sqrt{\frac{2\varepsilon_i}{L}}\right) \; , \quad
    \widetilde{B}_k = \sum_{i=1}^k \frac{i^2 \varepsilon_i }{L} \; .
    \]
\end{proposition}
In this case, we require the series $\{k\|e_k\|\}$ and $\{k\sqrt{\varepsilon_k}\}$ to be summable to achieve the optimal $O(1/k^2)$ rate, which is an (unsurprisingly) stronger constraint than in the basic case. A sufficient condition is for $\|e_k\|$ and $\sqrt{\varepsilon_k}$ to decrease as $O(1/k^{2+\delta})$ for any $\delta > 0$. Note that, as opposed to Proposition~\ref{prop:convex} that is stated for the average iterate, this bound is for the last iterate $x_k$. 

Again, it is interesting to see what happens when the summability assumption is not met. First, if $\|e_k\|$ or $\sqrt{\varepsilon_k}$ decreases at a rate of $O(1/k^2)$, then $k (\|e_k\| + \sqrt{e_k})$ decreases as $O(1/k)$ and $\widetilde{A}_k$ grows as $O(\log k)$ (note that $\widetilde{B}_k$ is always smaller than $\widetilde{A}_k$), yielding a convergence rate of $O\left(\frac{\log^2 k}{k^2}\right)$ for $f(x_k)-f(x^*)$. Also, and perhaps more interestingly, if $\|e_k\|$ or $\sqrt{\varepsilon_k}$ decreases at a rate of $O(1/k)$, Eq.~(\ref{eq:convex_accelerated_bound}) does not guarantee convergence of the function values. More generally, the form of $\widetilde{A}_k$ and $\widetilde{B}_k$ indicates that errors have a greater effect on the accelerated method than on the basic method. Hence, as also discussed in~\cite{devolder2010first}, unlike in the error-free case, the accelerated method may not necessarily be better than the basic method because it is more sensitive to errors in the computation.

In the case where $g$ is \emph{strongly} convex it is possible to obtain linear convergence rates that depend on the ratio
\[
\gamma = \mu/L,
\] 
as opposed to the sublinear convergence rates discussed above.  In particular, we obtain the following convergence rate on the iterates of the basic proximal-gradient method:
\begin{proposition}[Basic proximal-gradient method - Strong convexity]
\label{prop:strongly_convex}
Assume (\textbf{H}), that $g$ is $\mu$-strongly convex, and that we iterate recursion~\eqref{eq:2} with $y_k = x_k$.
Then, for all $k \geqslant 1$, we have:
\BEQ
\|x_k - x^\ast\| \leqslant \left(1 - \gamma\right)^k (\|x_0 - x^\ast\| + \bar{A}_k) \; ,
\EEQ
with
\[
\quad \bar{A}_k = \sum_{i=1}^k \left(1 - \gamma\right)^{-i}\left(\frac{\|  e_i \|}{L} +
\sqrt{\frac{2\varepsilon_i}{L}} \right)\; .
\]
\end{proposition}
A consequence of this proposition is that we obtain a linear rate of convergence even in the presence of errors, provided that $\|e_k\|$ and $\sqrt{\varepsilon_k}$ decrease linearly to 0. If they do so at a rate of $Q' < \left(1 - \gamma\right)$, then the convergence rate of $\|x_k - x^\ast\|$ is linear with constant $\left(1 - \gamma\right)$, as in the error-free algorithm. If we have $Q' > \left(1 - \gamma\right)$, then the convergence of $\|x_k - x^\ast\|$ is linear with constant $Q'$. If we have $Q' = \left(1 - \gamma\right)$, then $\|x_k - x^\ast\|$ converges to 0 as $O(k \left(1 - \gamma\right)^k) = o\left(\left[\left(1 - \gamma\right) + \delta'\right]^k\right)$ for all $\delta' > 0$.

Finally, we consider the accelerated proximal-gradient algorithm when $g$ is strongly convex.  We focus on a basic variant of the algorithm where $\beta_k$ is set to $(1-\sqrt{\gamma})/(1+\sqrt{\gamma})$~\cite[\S2.2.1]{nesterov2004introductory}:
\begin{proposition}[Accelerated proximal-gradient method - Strong convexity]
\label{prop:accelerated_strongly_convex}
Assume (\textbf{H}), that $g$ is $\mu$-strongly convex, and that we iterate recursion~\eqref{eq:2} with $y_k = x_k + \frac{1-\sqrt{\gamma}}{1+\sqrt{\gamma}}(x_k - x_{k-1})$. Then, for all $k \geqslant 1$, we have
\BEQ
\label{eq:strongly_convex_accelerated_bound}
f(x_k)-f(x^*) \leqslant \left(1 - \sqrt{\gamma}\right)^k\left(\sqrt{2(f(x_0) - f(x^*))} + \widehat{A}_k\sqrt{\frac{2}{\mu}} + \sqrt{\widehat{B}_k}\right)^2,
\EEQ
with 
\[
\widehat{A}_k = \sum_{i=1}^k \left(\|  e_i \| + \sqrt{2L\varepsilon_i}\right)\left(1 - \sqrt{\gamma}\right)^{-i/2} \; , \quad
   \widehat{B}_k =\sum_{i=1}^k \varepsilon_i\left(1 - \sqrt{\gamma}\right)^{-i} \; .
   \]
\end{proposition}
Note that while we have stated the result in terms of function values, we obtain an analogous result on the iterates because by strong convexity of $f$ we have
\[
\frac{\mu}{2}||x_k - x_*||^2 \leq f(x_k)-f(x_*).
\]
This proposition implies that we obtain a linear rate of convergence in the presence of errors provided that $||e_k||^2$ and $\ve_k$ decrease linearly to 0.  If they do so at a rate $Q' < (1-\sqrt{\gamma})$, then the constant is $(1-\sqrt{\gamma})$, while if $Q' > (1-\sqrt{\gamma})$ then the constant will be $Q'$. 
Thus, the accelerated inexact proximal-gradient method will have a faster convergence rate than the \emph{exact} basic proximal-gradient method provided that $Q' < (1-\gamma)$.  Oddly, in our analysis of the strongly convex case, the accelerated method is \emph{less sensitive} to errors than the basic method.  However, unlike the basic method, the accelerated method requires knowing $\mu$ in addition to $L$.  If $\mu$ is misspecified, then the convergence rate of the accelerated method may be slower than the basic method.

\section{Experiments}
\label{sec:experiments}

We tested the basic inexact proximal-gradient and accelerated proximal-gradient methods on the CUR-like factorization optimization problem introduced in~\cite{mairal2011convex} to approximate a given matrix $W$,
\[
\min_X \half\norm{W-WXW}_F^2 + \lambda_{\textrm{row}}\sum_{i=1}^{n_r}||X^i||_p + \lambda_{col}\sum_{j=1}^{n_c}||X_j||_p \; .
\]
Under an appropriate choice of $p$, this optimization problem yields a matrix $X$ with sparse rows \emph{and} sparse columns, meaning that entire rows and columns of the matrix $X$ are set to exactly zero.  In~\cite{mairal2011convex}, the authors used an accelerated proximal-gradient method and chose $p=\infty$ since under this choice the proximity operator can be computed exactly.  However, this has the undesirable effect that it also encourages all values in the same row (or column) to have the same magnitude.  The more natural choice of $p=2$ was not explored since in this case there is no known algorithm to exactly compute the proximity operator.  

Our experiments focused on the case of $p=2$.  In this case, it is possible to very quickly compute an approximate proximity operator using the block coordinate descent (BCD) algorithm presented in~\cite{jenatton2010proximal}, which is equivalent to the proximal variant of Dykstra's algorithm introduced by~\cite{bauschke2008dykstra}.  In our implementation of the BCD method, we alternate between computing the proximity operator with respect to the rows and to the columns.  Since the BCD method allows us to compute a duality gap when solving the proximal problem, we can run the method until the duality gap is below a given error threshold $\ve_k$ to find an $x_{k+1}$ satisfying~\eqref{eq:3}.

In our experiments, we used the four data sets examined by~\cite{mairal2011convex}\footnote{The datasets are freely available at \url{http://www.gems-system.org}.} and we choose $\lambda_{row} = .01$ and $\lambda_{col} = .01$, which yielded approximately 25--40\% non-zero entries in $X$ (depending on the data set).
Rather than assuming we are given the Lipschitz constant $L$, on the first iteration we set $L$ to $1$ and following~\cite{nesterov2007gradient} we double our estimate anytime $g(x_k) > g(y_{k-1}) + \langle g'(y_{k-1}),x_k-y_{k-1}\rangle + (L/2)||x_k-y_{k-1}||^2$.
We tested three different ways to terminate the approximate proximal problem, each parameterized by a parameter $\alpha$:
\begin{itemize}
\item $\varepsilon_k = 1/k^\alpha$: Running the BCD algorithm until the duality gap is below $1/k^\alpha$.
\item $\varepsilon_k = \alpha$: Running the BCD algorithm until the duality gap is below $\alpha$.
\item $n = \alpha$: Running the BCD algorithm for a fixed number of iterations $\alpha$.
\end{itemize}
Note that all three strategies lead to global convergence in the case of the basic proximal-gradient method, the first two give a convergence rate up to some fixed optimality tolerance, and in this paper we have shown that the first one (for large enough $\alpha$) yields a convergence rate for an arbitrary optimality tolerance.  Note that the iterates produced by the BCD iterations are \emph{sparse}, so we expected the algorithms to spend the majority of their time solving the proximity problem.  Thus, we used the function value against the number of BCD iterations as a measure of performance.  We plot the results after $500$ BCD iterations for the four data sets 
for the proximal-gradient method in Figure~\ref{fig:gradient}, and the accelerated proximal-gradient method in Figure~\ref{fig:nesterov}.  
In these plots, the first column varies $\alpha$ using the choice $\ve_k = 1/k^\alpha$, the second column varies $\alpha$ using the choice $\ve_k = \alpha$, and the third column varies $\alpha$ using the choice $n = \alpha$.  We also include one of the best methods from the first column in the second and third columns as a reference.

In the context of proximal-gradient methods the choice of $\ve_k = 1/k^3$, which is one choice that achieves the fastest convergence rate according to our analysis, gives the best performance across all four data sets.  However, in these plots we also see that reasonable performance can be achieved by any of the three strategies above provided that $\alpha$ is chosen carefully.  For example, choosing $n=3$ or choosing $\ve_k = 10^{-6}$ both give reasonable performance.  However, these are only empirical observations for these data sets and they may be ineffective for other data sets or if we change the number of iterations, while we have given theoretical justification for the choice $\ve_k = 1/k^3$.  

Similar trends are observed for the case of accelerated proximal-gradient methods, though the choice of $\ve_k = 1/k^3$ (which no longer achieves the fastest convergence rate according to our analysis) no longer dominates the other methods in the accelerated setting.  For the \emph{SRBCT} data set the choice $\ve_k = 1/k^4$, which is a choice that achieves the fastest convergence rate up to a poly-logarithmic factor, yields better performance than $\ve_k = 1/k^3$.  Interestingly, the only choice that yields the fastest possible convergence rate ($\ve_k = 1/k^5$) had reasonable performance but did not give the best performance on any data set. This seems to reflect the trade-off between performing \emph{inner} BCD iterations to achieve a small duality gap and performing \emph{outer} gradient iterations to decrease the value of $f$. Also, the constant terms which were not taken into account in the analysis do play an important role here, due to the relatively small number of \emph{outer} iterations performed.

\begin{figure}
\centering
\includegraphics[width=.32\columnwidth]{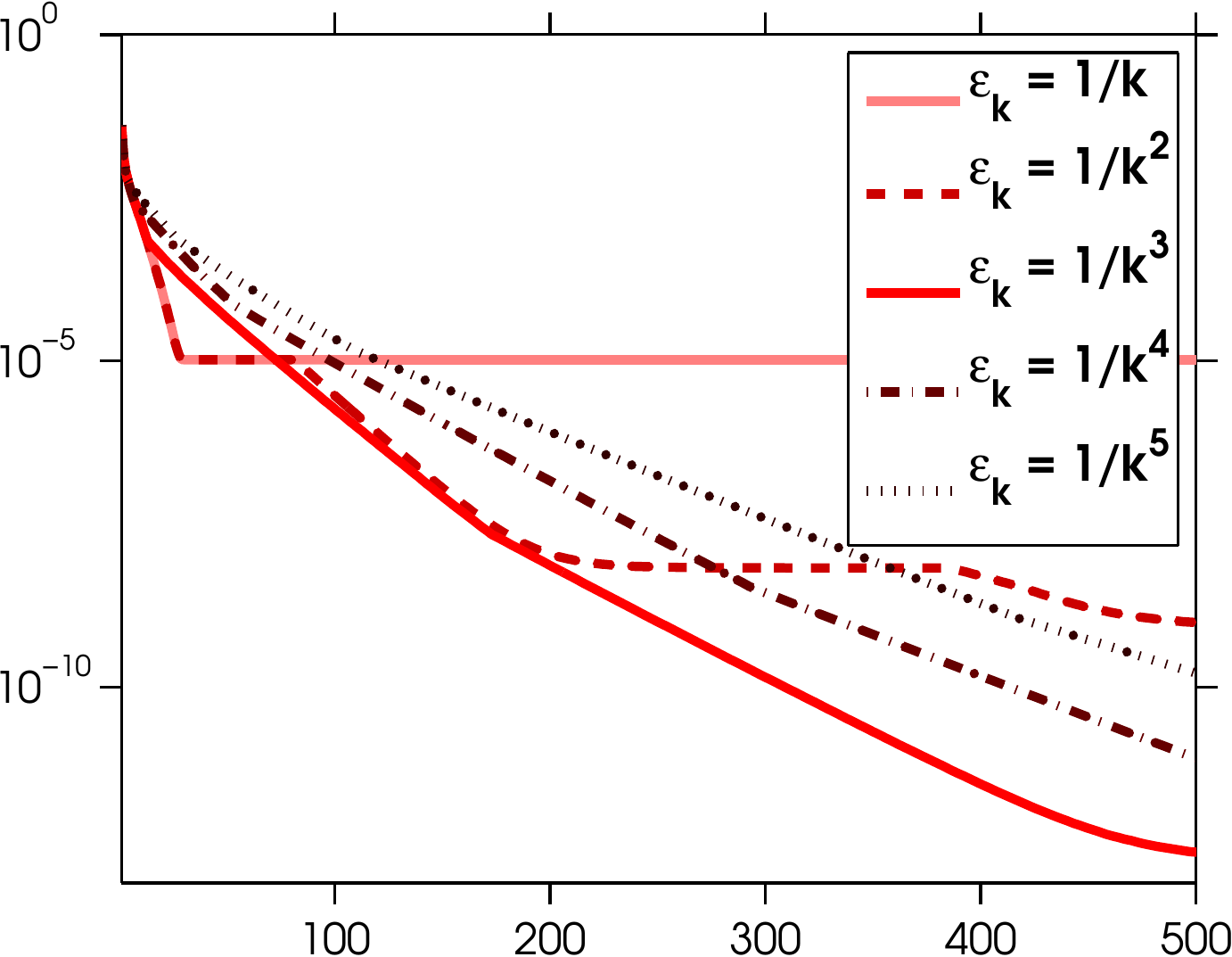}
\includegraphics[width=.32\columnwidth]{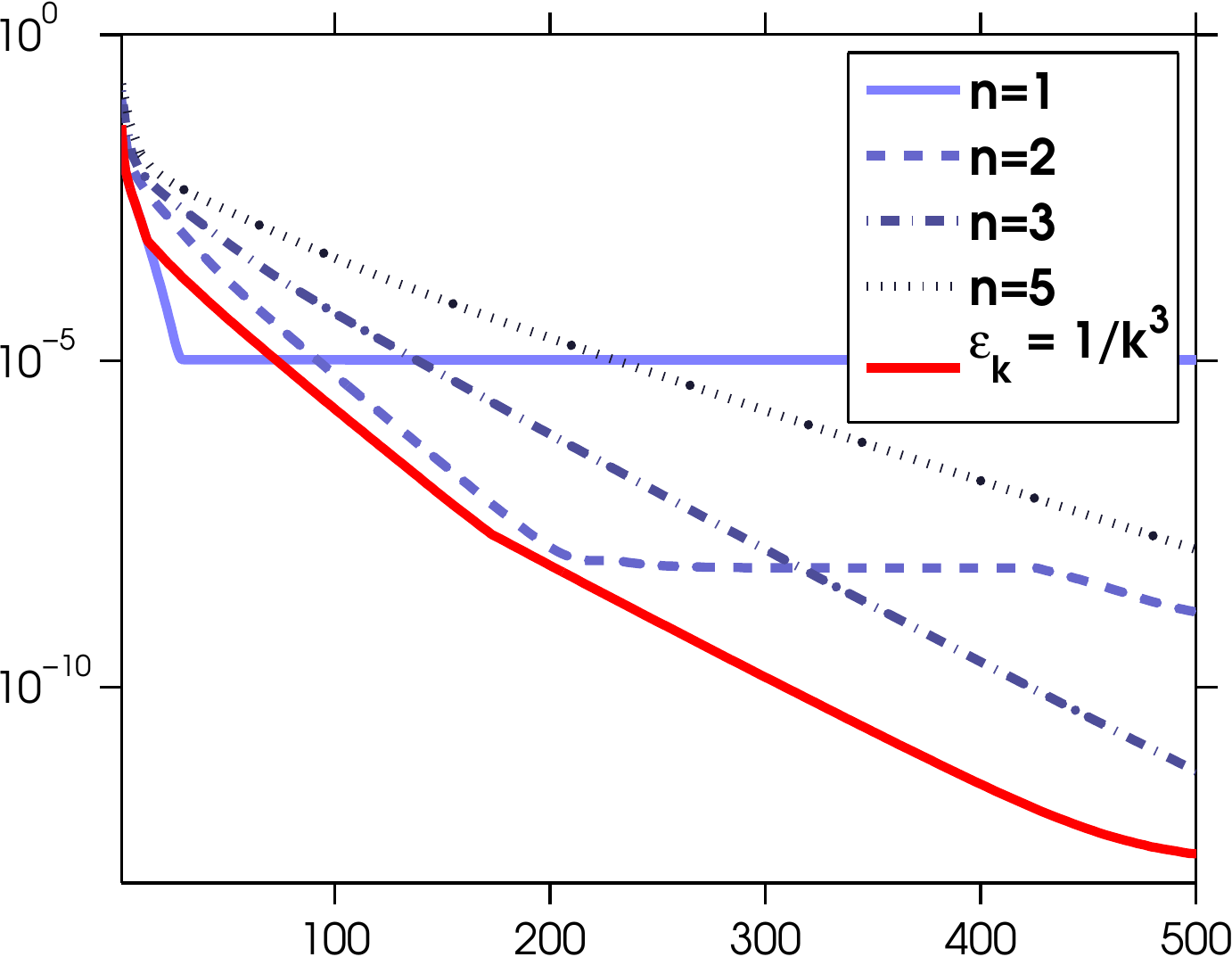}
\includegraphics[width=.32\columnwidth]{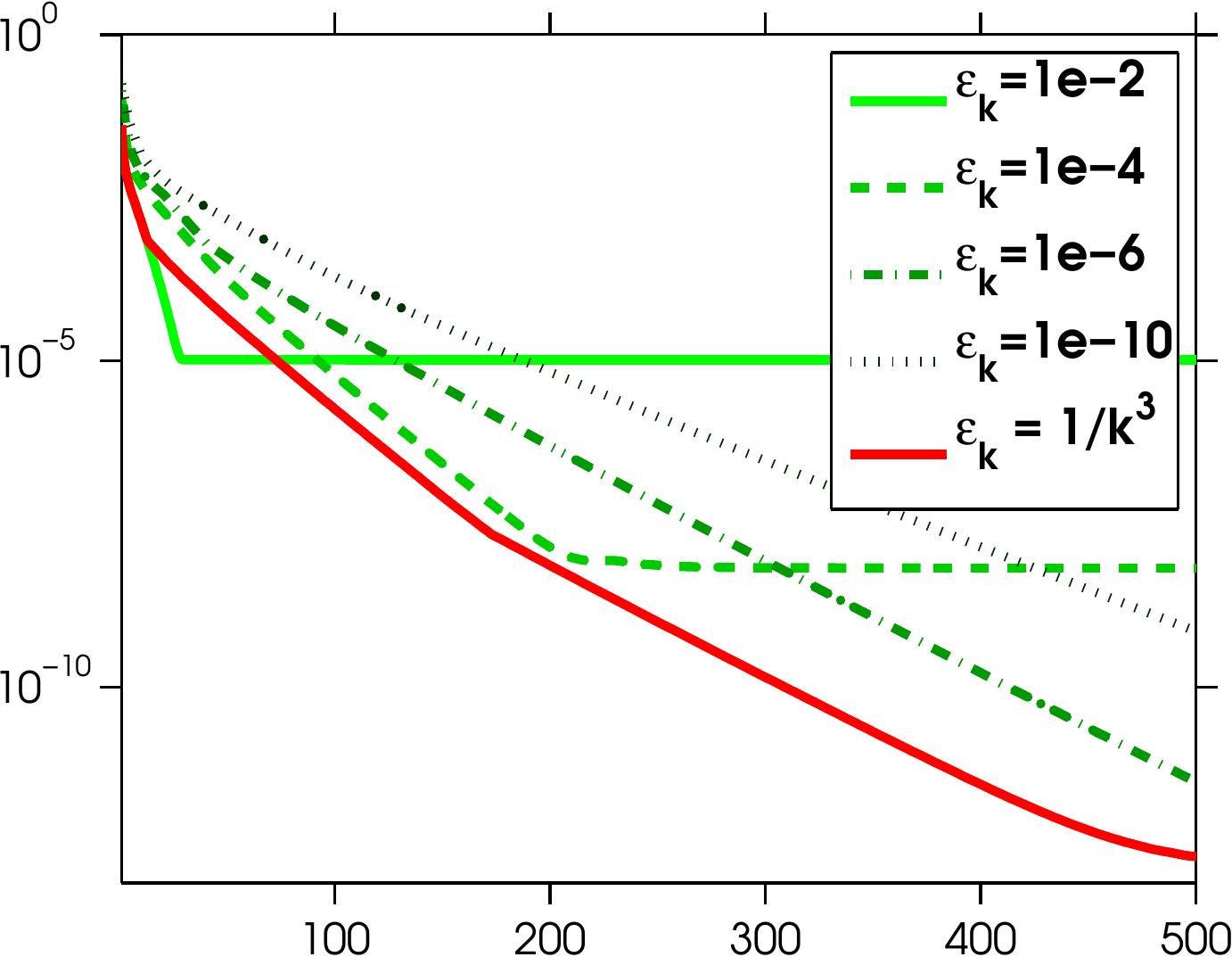}
\includegraphics[width=.32\columnwidth]{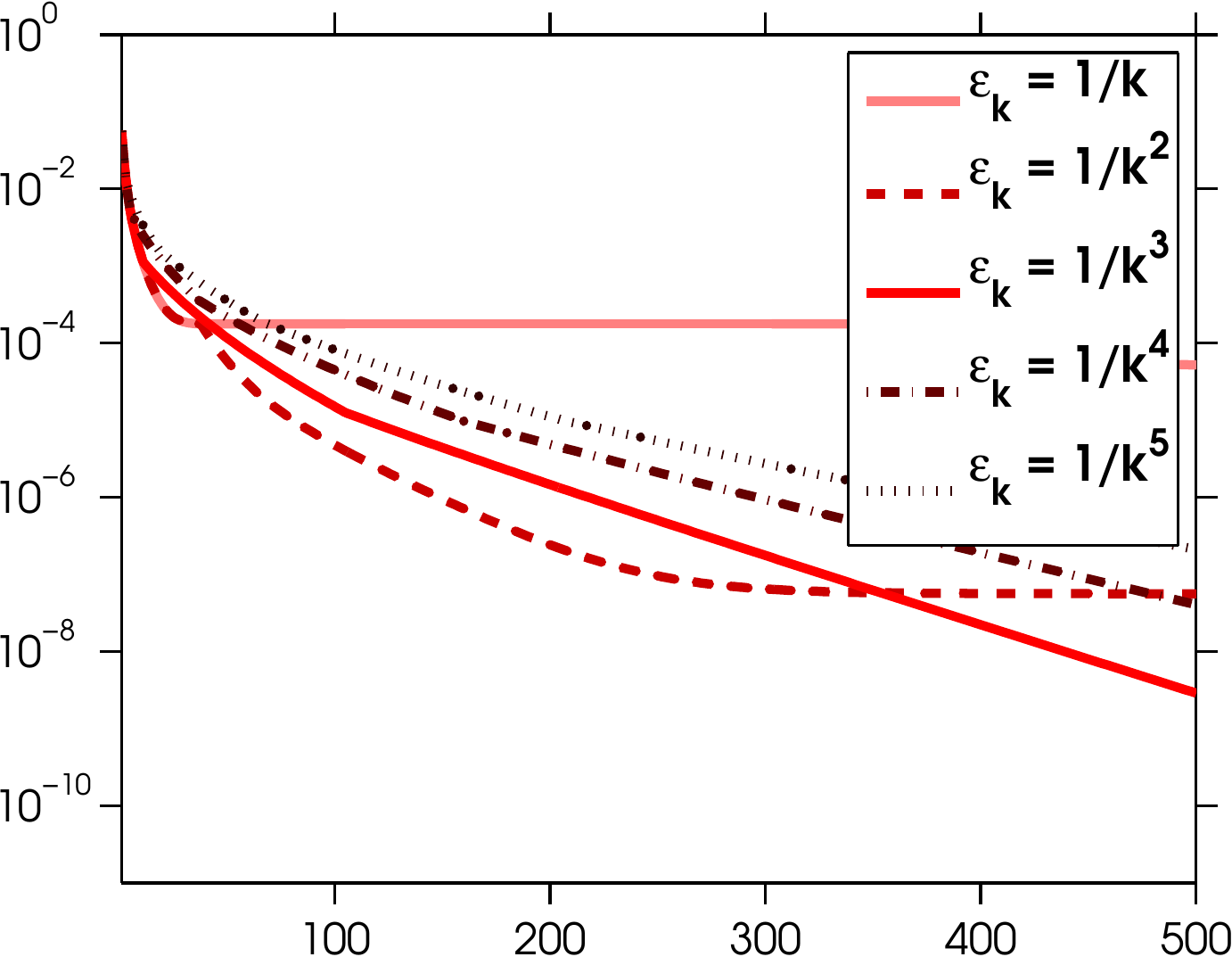}
\includegraphics[width=.32\columnwidth]{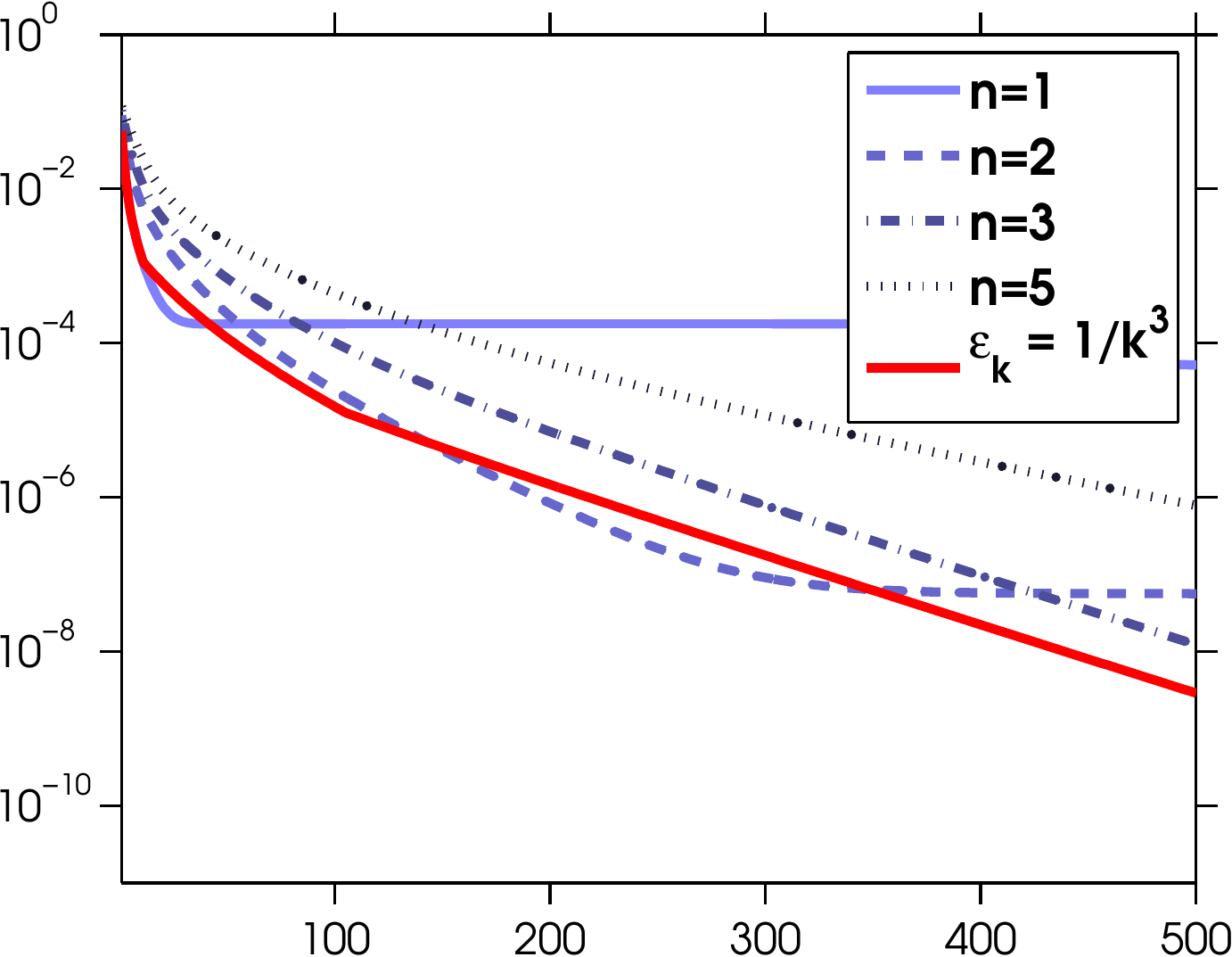}
\includegraphics[width=.32\columnwidth]{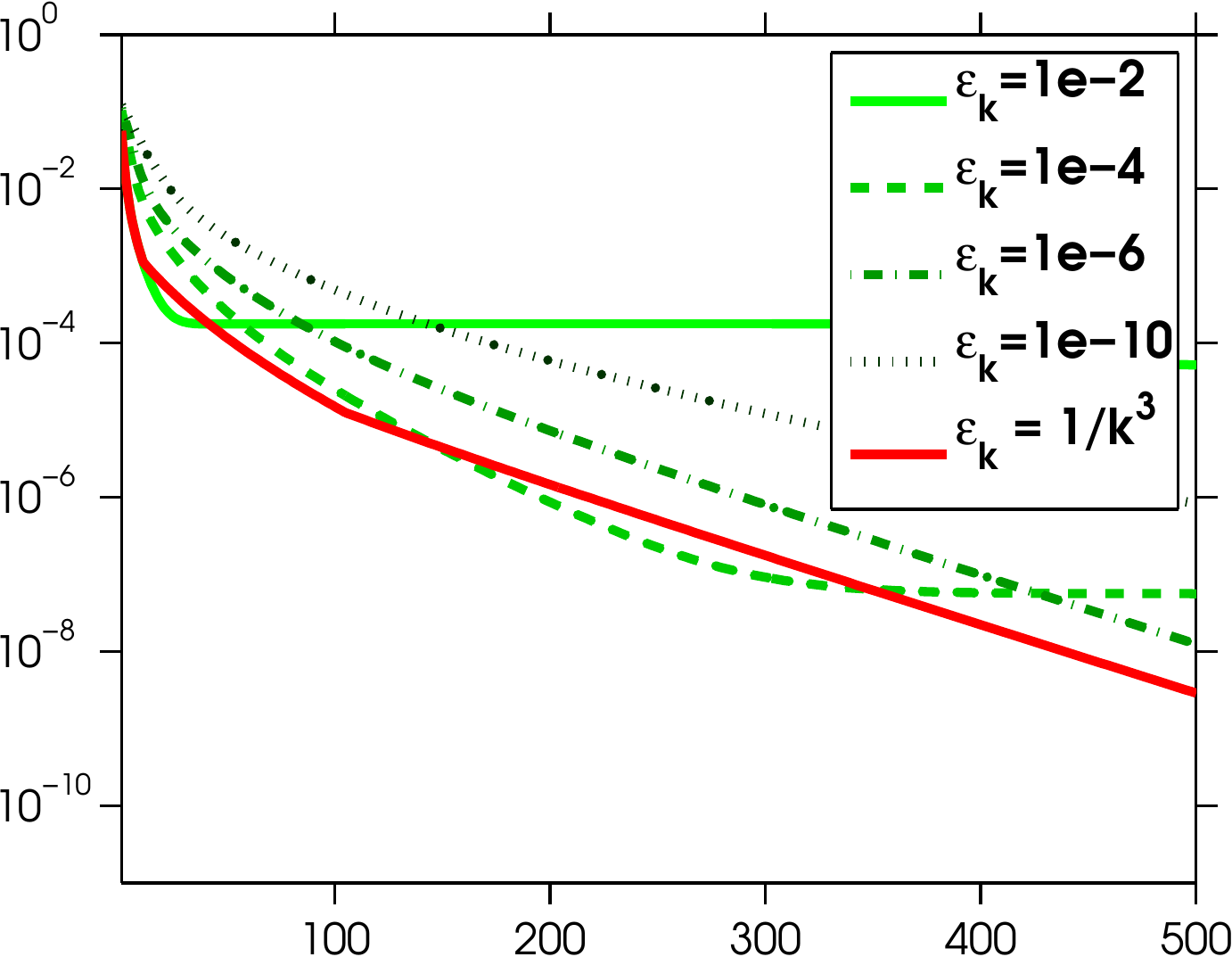}
\includegraphics[width=.32\columnwidth]{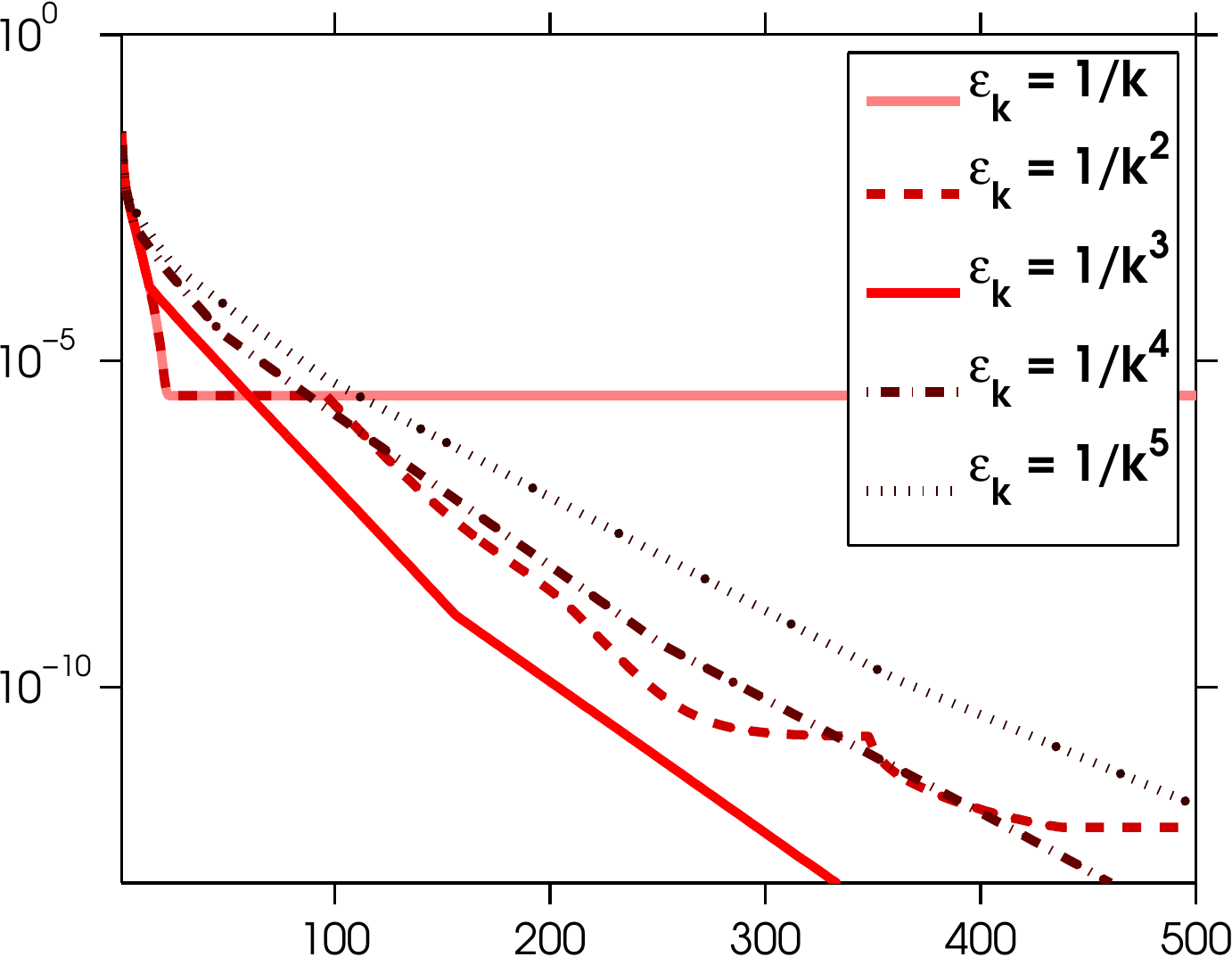}
\includegraphics[width=.32\columnwidth]{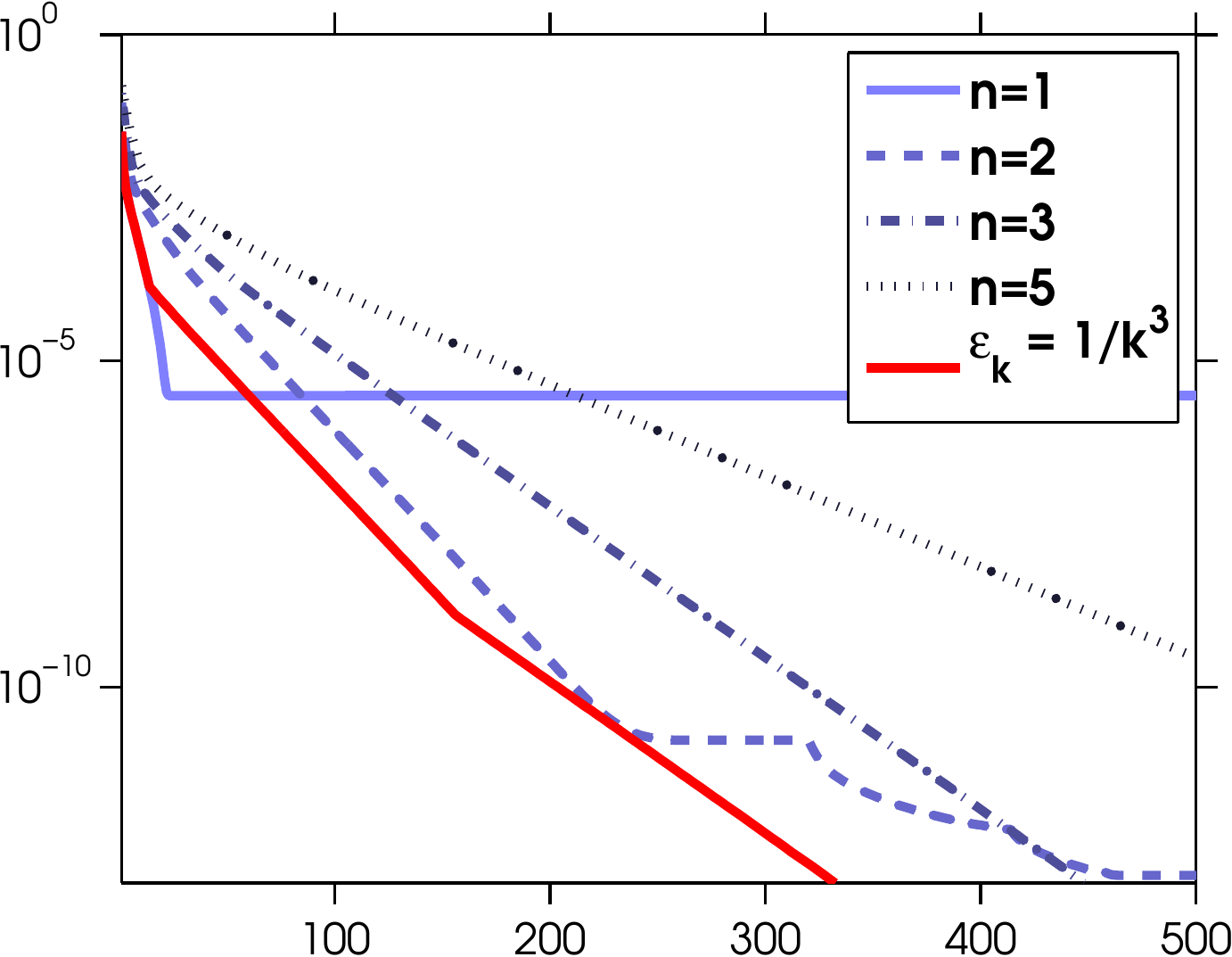}
\includegraphics[width=.32\columnwidth]{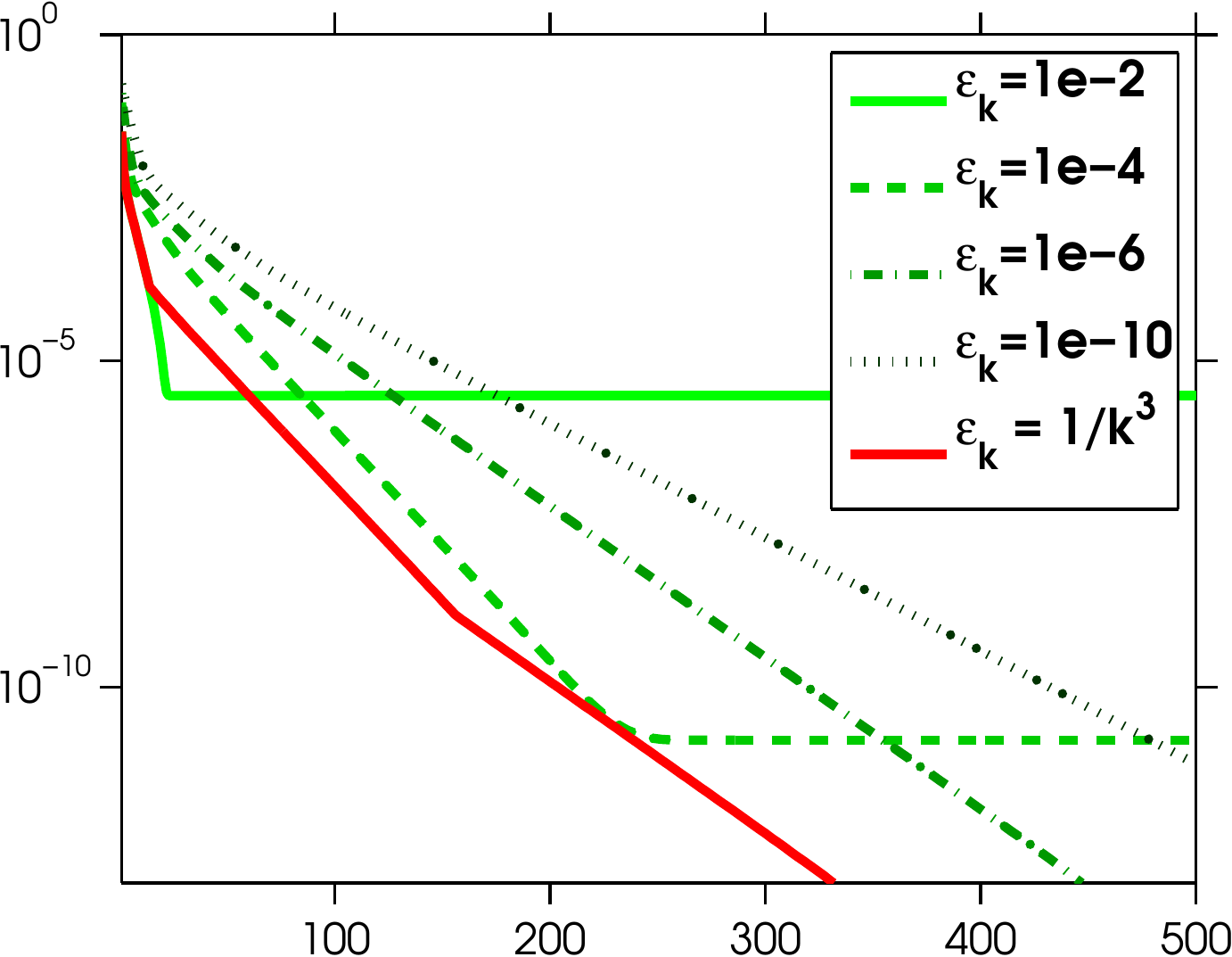}
\includegraphics[width=.32\columnwidth]{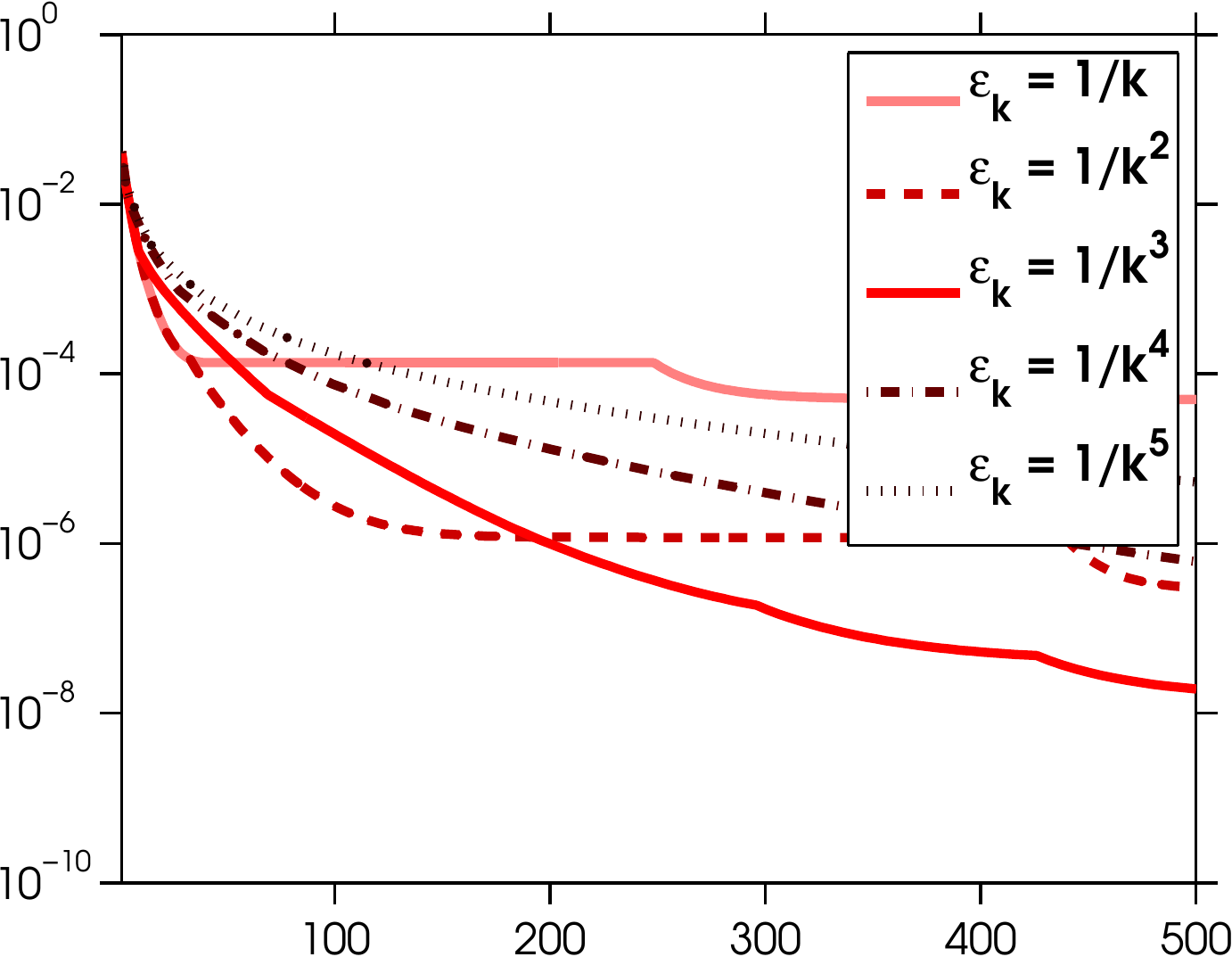}
\includegraphics[width=.32\columnwidth]{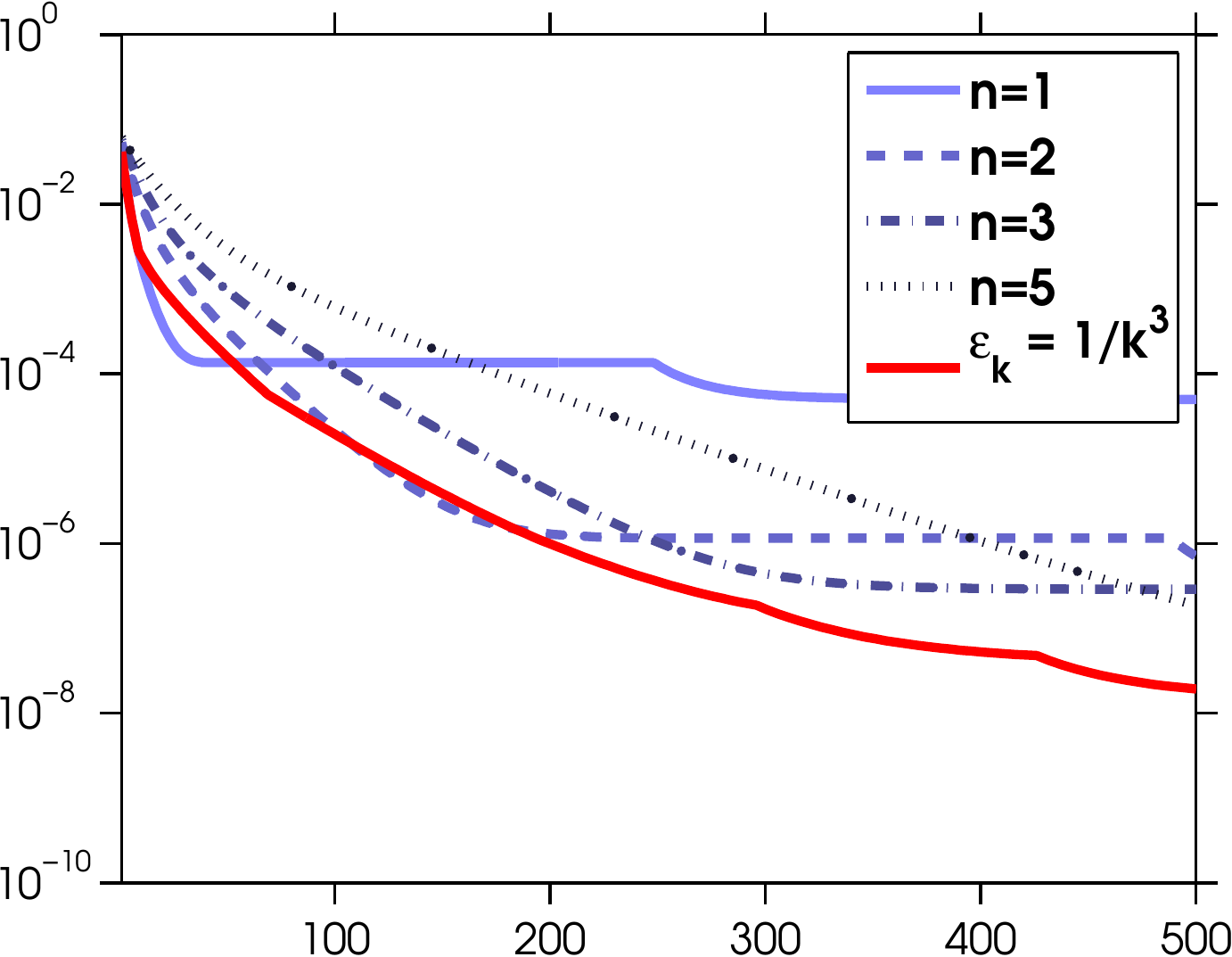}
\includegraphics[width=.32\columnwidth]{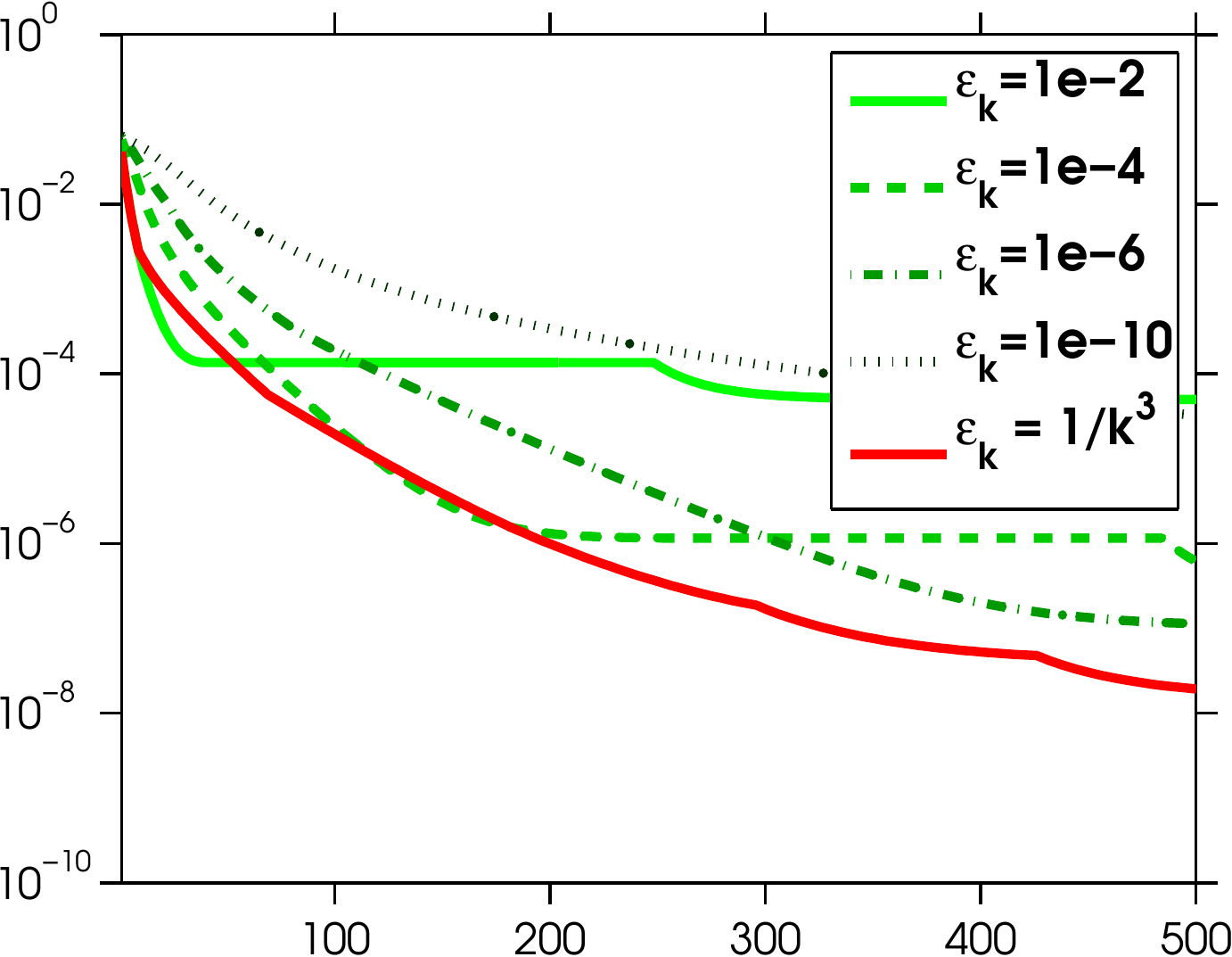}
\caption{Objective function against number of proximal iterations for the proximal-gradient method with different strategies for terminating the approximate proximity calculation.  From top to bottom we have the \emph{9\_Tumors}, \emph{Brain\_Tumor1}, \emph{Leukemia1}, and \emph{SRBCT} data sets.}
\label{fig:gradient}
\end{figure}

\begin{figure}
\centering
\includegraphics[width=.32\columnwidth]{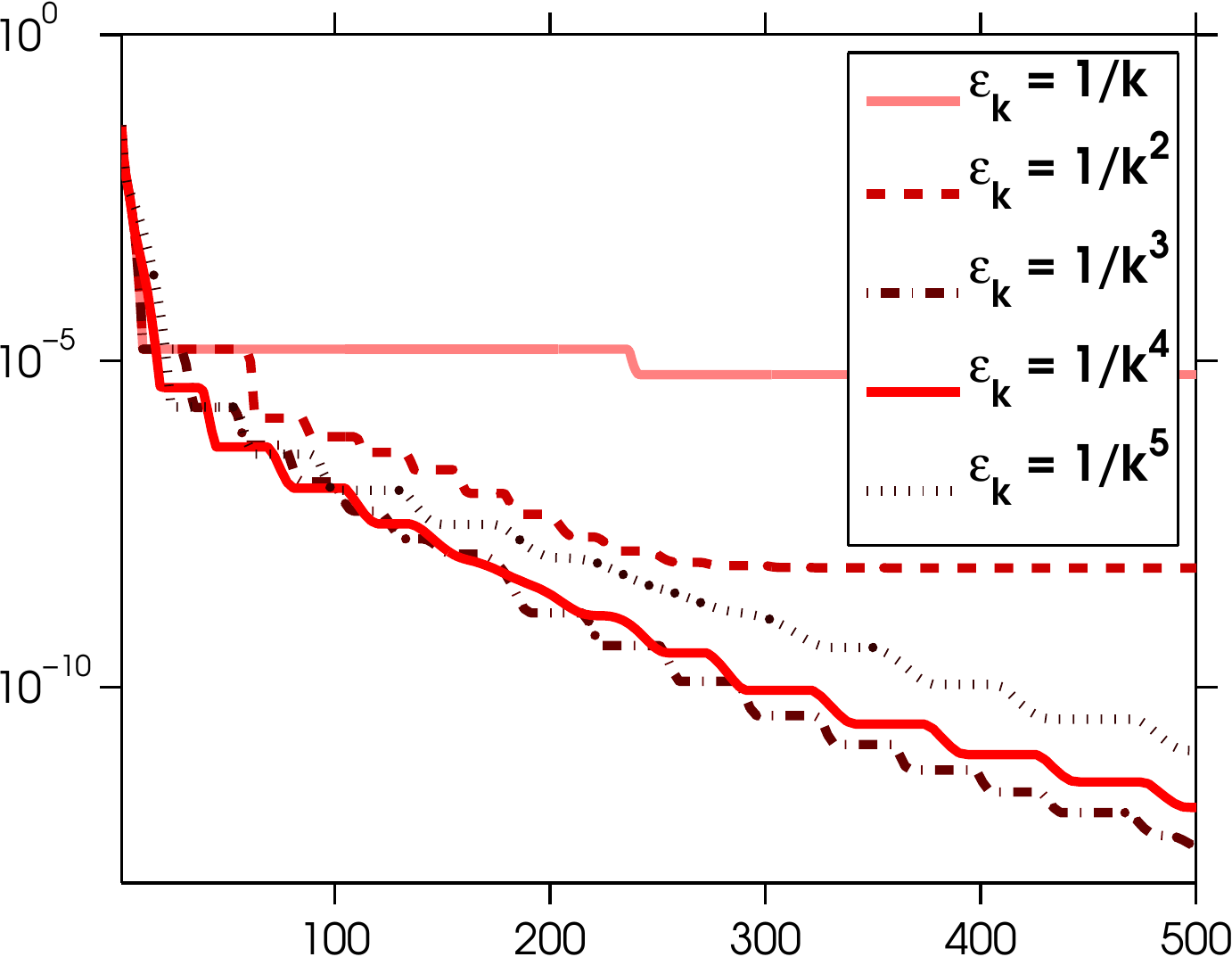}
\includegraphics[width=.32\columnwidth]{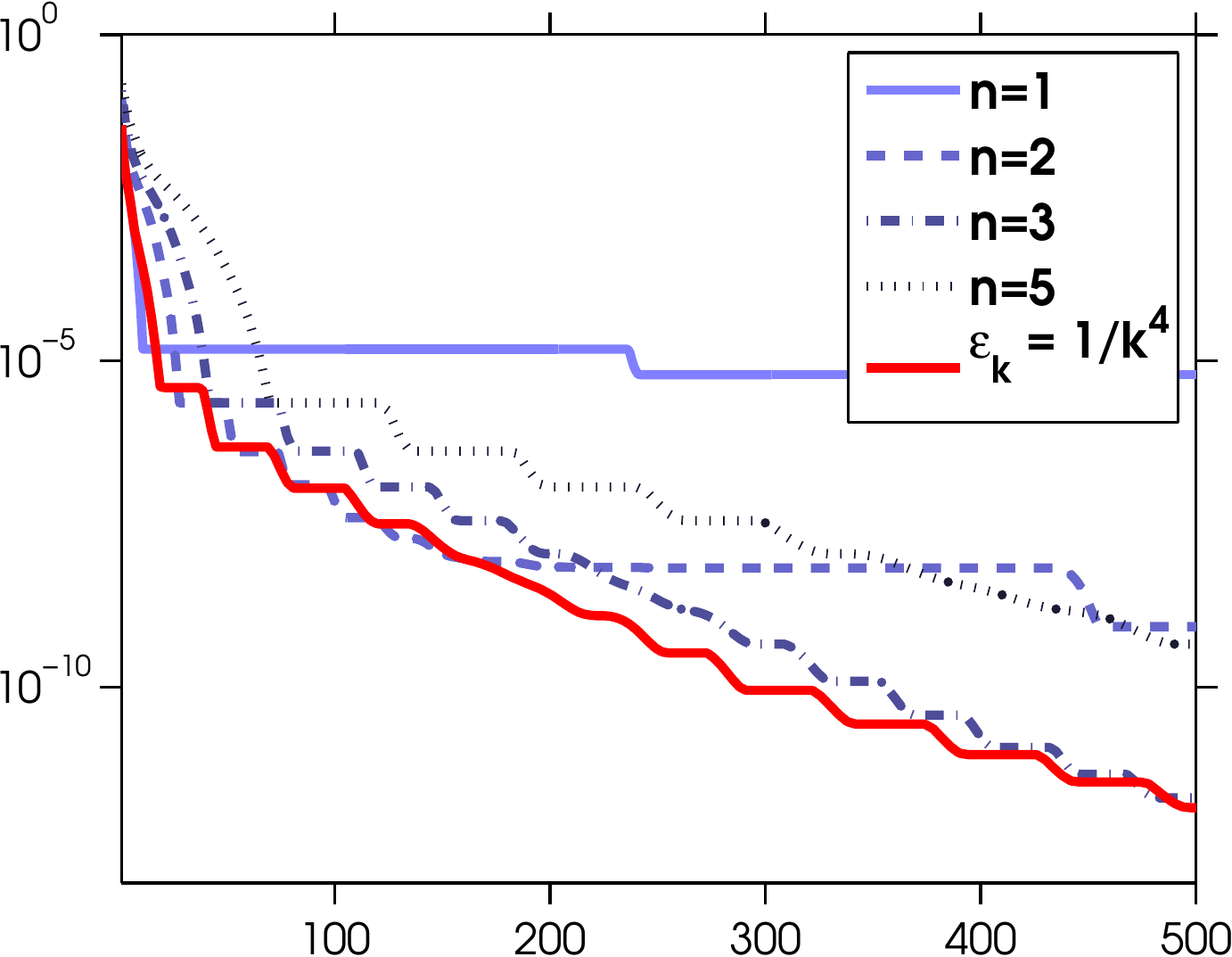}
\includegraphics[width=.32\columnwidth]{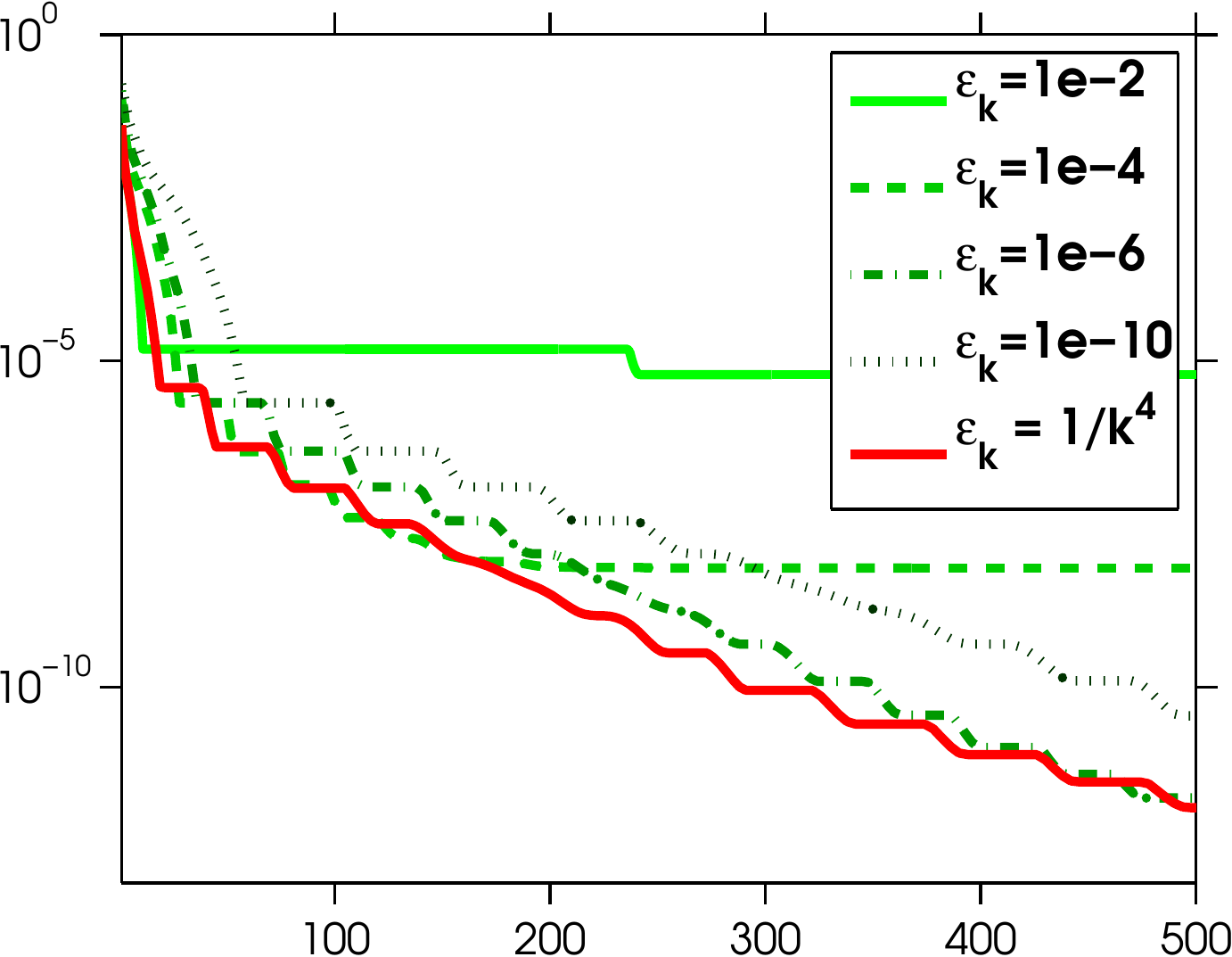}
\includegraphics[width=.32\columnwidth]{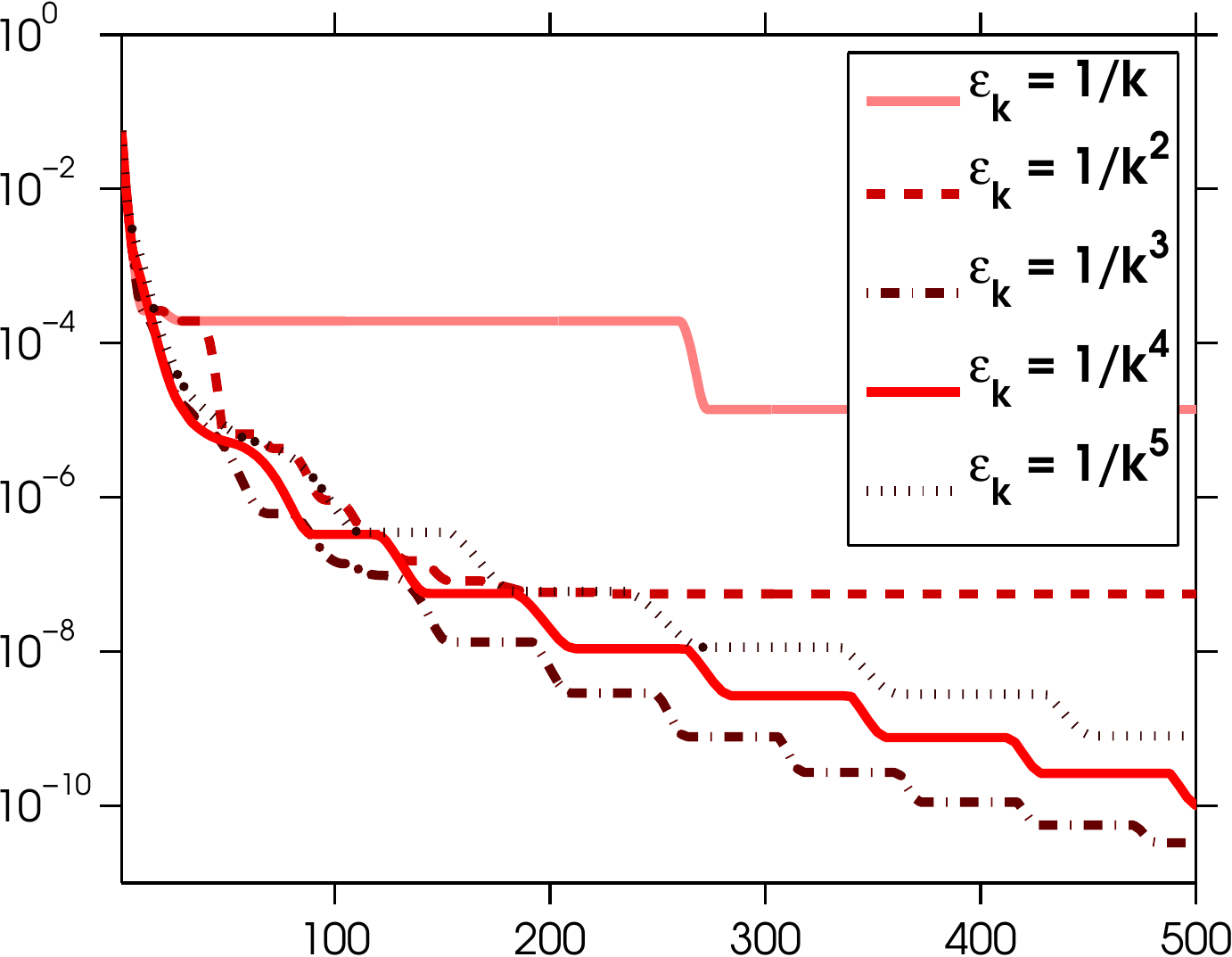}
\includegraphics[width=.32\columnwidth]{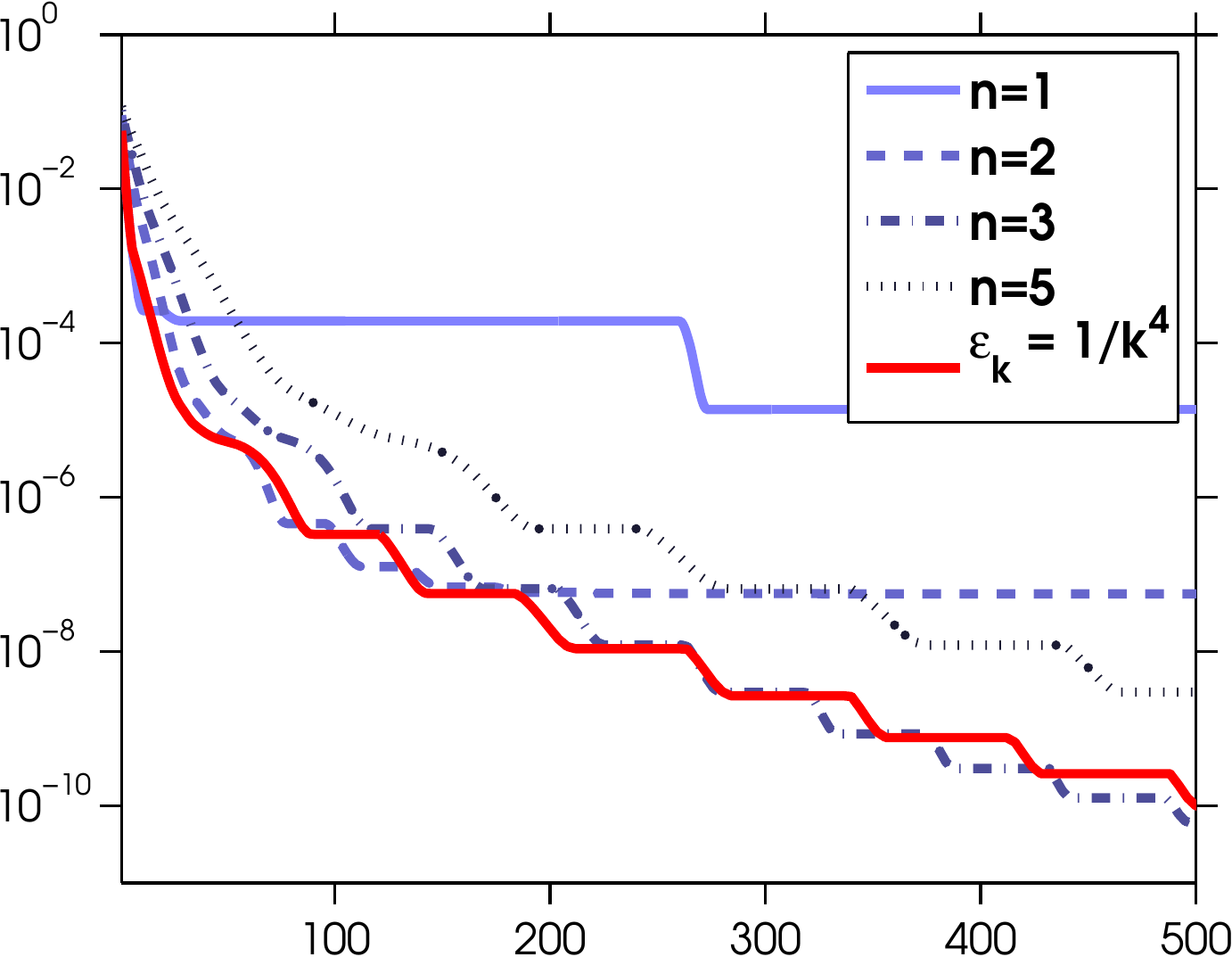}
\includegraphics[width=.32\columnwidth]{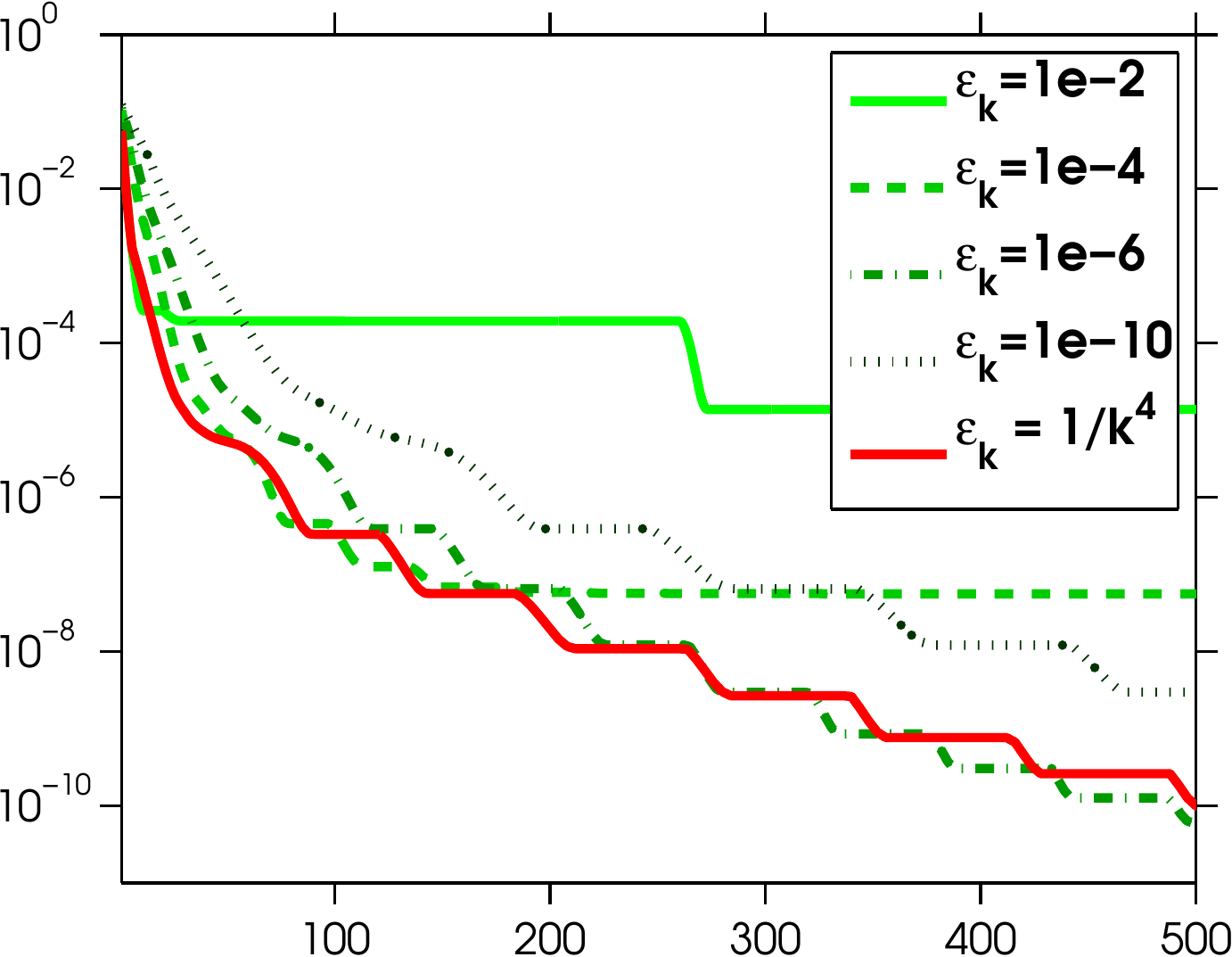}
\includegraphics[width=.32\columnwidth]{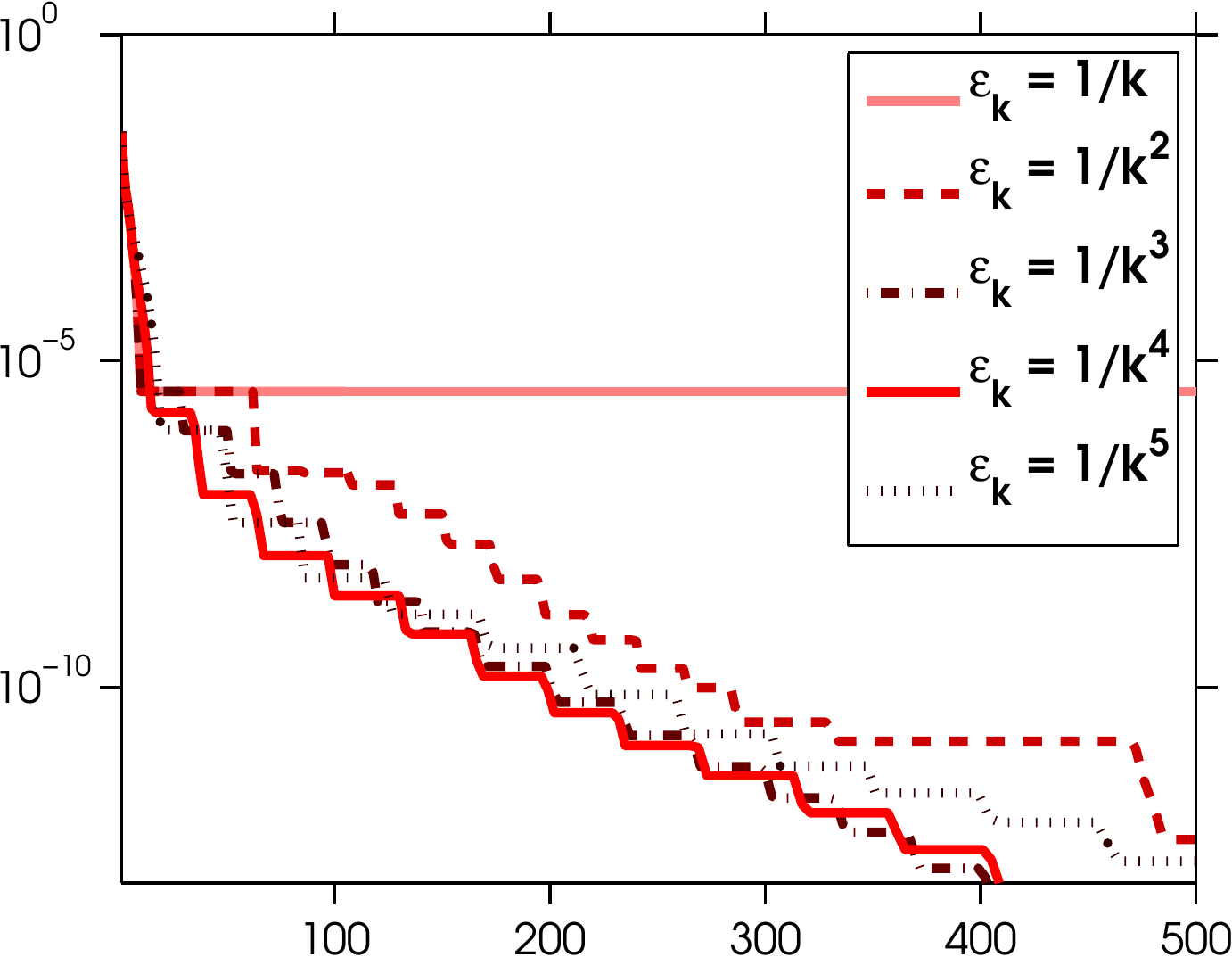}
\includegraphics[width=.32\columnwidth]{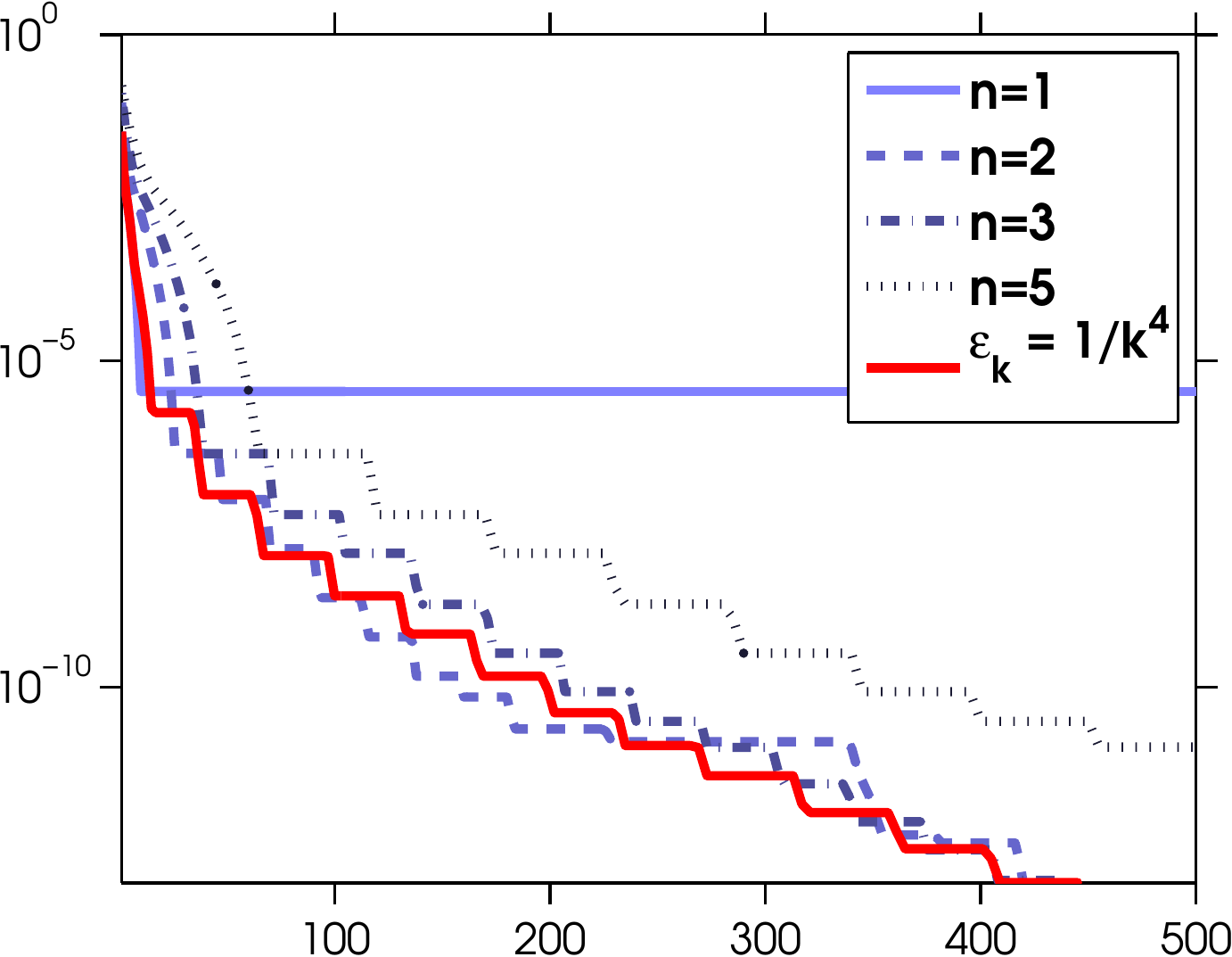}
\includegraphics[width=.32\columnwidth]{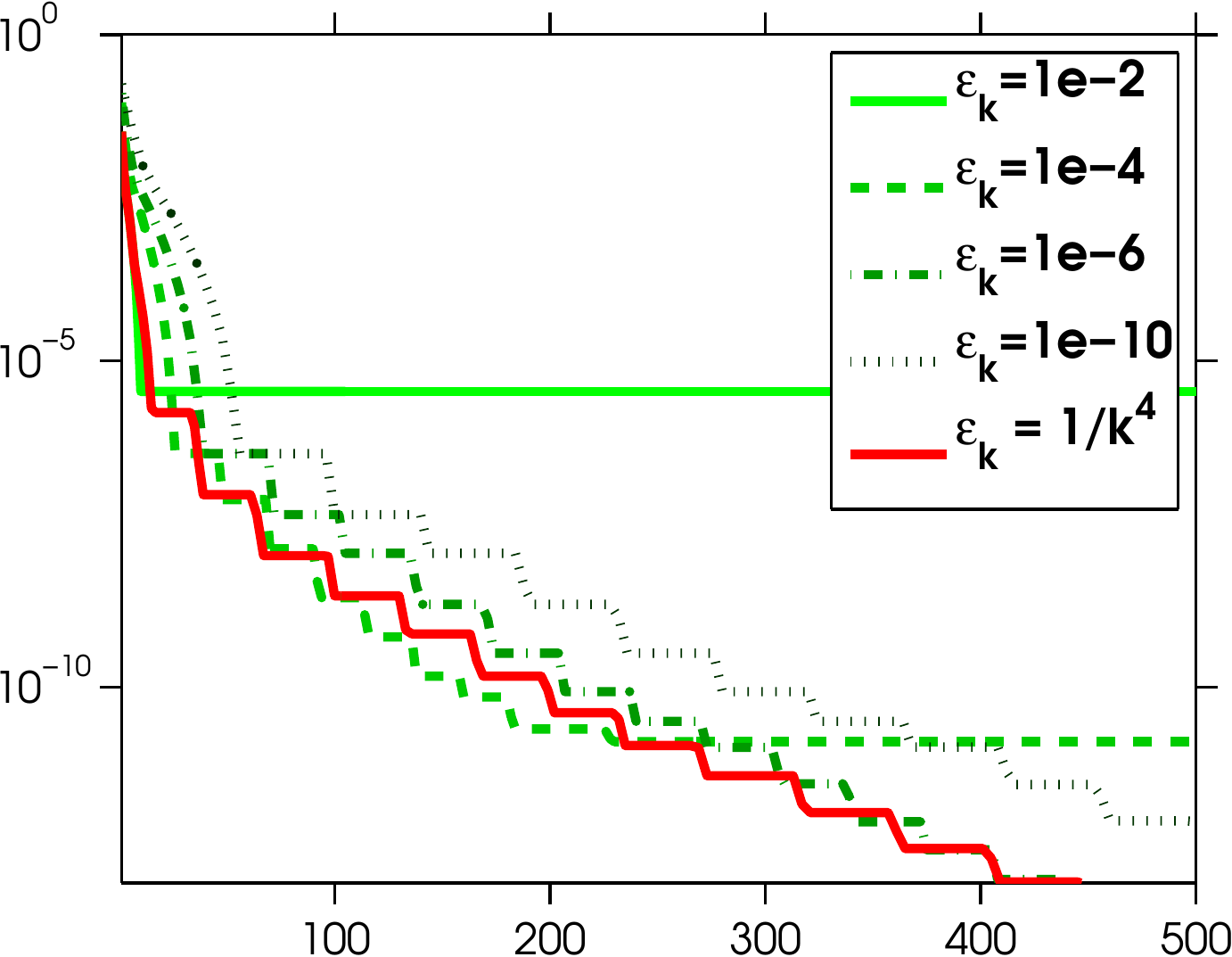}
\includegraphics[width=.32\columnwidth]{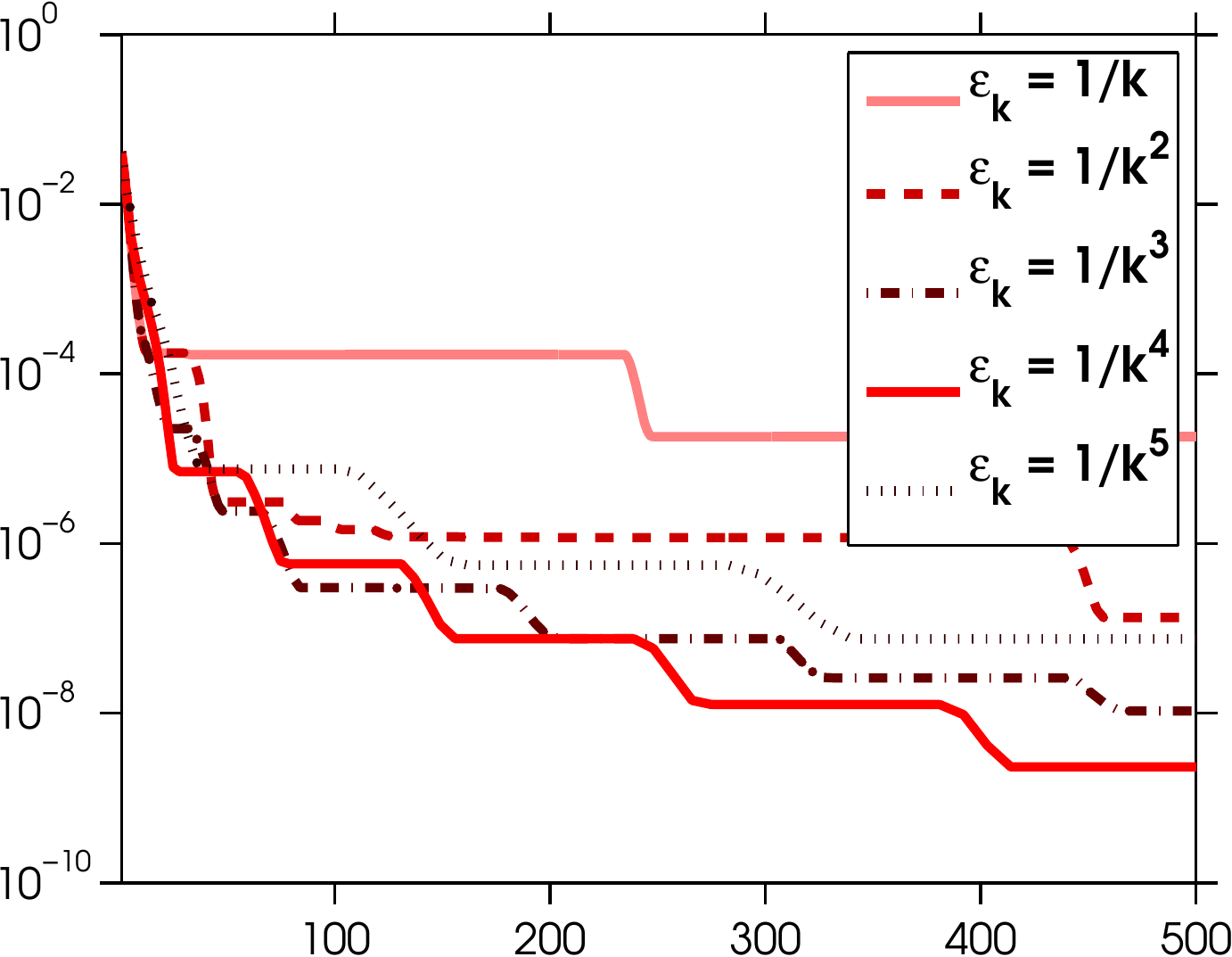}
\includegraphics[width=.32\columnwidth]{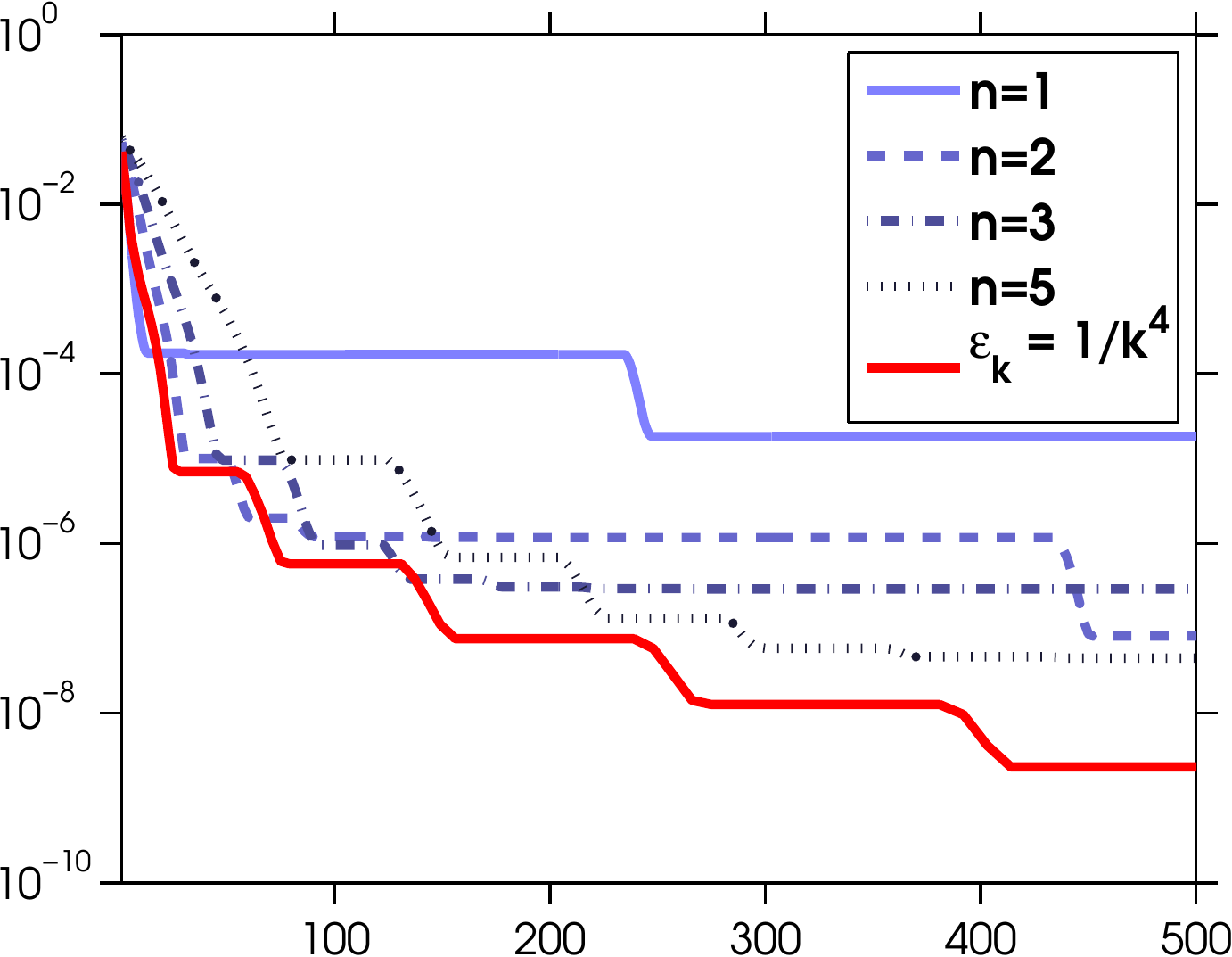}
\includegraphics[width=.32\columnwidth]{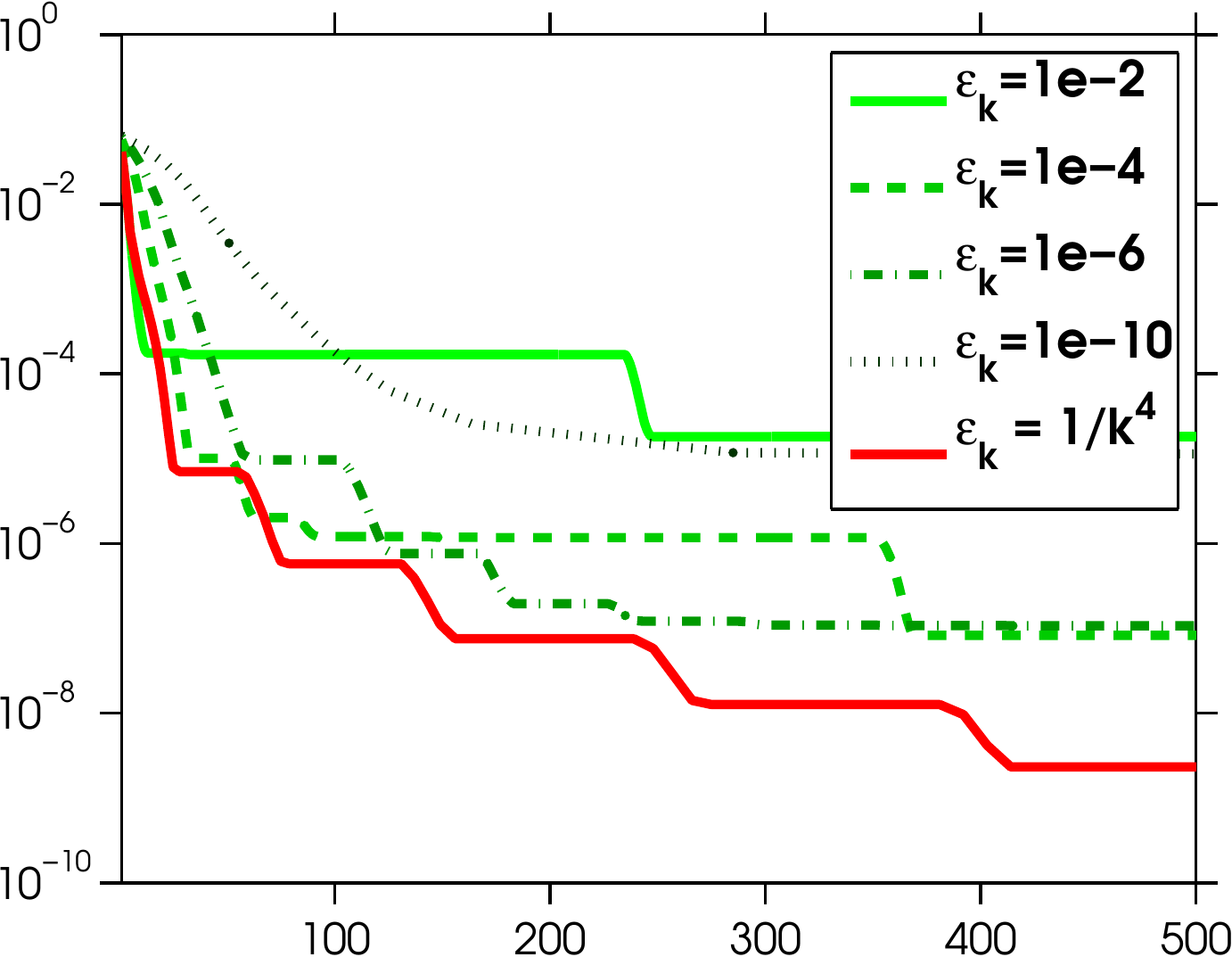}
\caption{Objective function against number of proximal iterations for the \emph{accelerated} proximal-gradient method with different strategies for terminating the approximate proximity calculation. From top to bottom we have the \emph{9\_Tumors}, \emph{Brain\_Tumor1}, \emph{Leukemia1}, and \emph{SRBCT} data sets.}
\label{fig:nesterov}
\end{figure}

\section{Discussion}
\label{sec:discussion}

An alternative to inexact proximal methods for solving structured sparsity problems are smoothing methods~\cite{nesterov2005smooth} and alternating direction methods~\cite{combettes2009proximal}.  However,  a major disadvantage of both these approaches is that the iterates are not sparse, so they can not take advantage of the sparsity of the problem when running the algorithm.  In contrast, the method proposed in this paper has the appealing property that it tends to generate sparse iterates.
Further, the accelerated smoothing method only has a convergence rate of $O(1/k)$, and the performance of alternating direction methods is often sensitive to the exact choice of their penalty parameter.  
On the other hand, while our analysis suggests using a sequence of errors like $O(1/k^\alpha)$ for $\alpha$ large enough, the practical performance of inexact proximal-gradients methods will be sensitive to the exact choice of this sequence.

Although we have illustrated the use of our results in the context of a structured sparsity problem, 
inexact proximal-gradient methods are also used in other applications such as total-variation~\cite{fadili-tip-10-tv,chen2010graph} and nuclear-norm~\cite{cai2008singular,ma2009fixed} regularization.  This work provides a theoretical justification for using inexact proximal-gradient methods in these and other applications, and suggests some guidelines for practioners that do not want to lose the appealing convergence rates of these methods.
Further, although our experiments and much of our discussion focus on errors in the calculation of the proximity operator, our analysis also allows for an error in the calculation of the gradient.  This may also be useful in a variety of contexts.  
For example, errors in the calculation of the gradient arise when fitting undirected graphical models and using an iterative method to approximate the gradient of the log-partition function~\cite{wainwright2003tree}.  Other examples include using a reduced set of training examples within kernel methods~\cite{kivinen2004online} or subsampling to solve semidefinite programming problems~\cite{alexandre2009subsampling}.

In our analysis, we assume that the smoothness constant $L$ is known, 
but it would be interesting to extend methods for estimating $L$ in the exact case~\cite{nesterov2007gradient} to the case of inexact algorithms.  
In the context of accelerated methods for strongly convex optimization, our analysis also assumes that $\mu$ is known, and it would be interesting to explore variants that do not make this assumption. 
We also note that if the basic proximal-gradient method is given knowledge of $\mu$, then our analysis can be modified to obtain a faster linear convergence rate of $(1-\gamma)/(1+\gamma)$ instead of $(1-\gamma)$ for strongly-convex optimization using a step size of $2/(\mu+L)$, see Theorem 2.1.15 of~\cite{nesterov2004introductory}.
Finally, we note that there has been recent interest in inexact proximal Newton-like methods~\cite{schmidt2011newtonlike}, and it would be interesting to analyze the effect of errors on the convergence rates of these methods.

\subsection*{Acknowledgements}
Mark Schmidt, Nicolas Le Roux, and Francis Bach are supported by the European Research Council (SIERRA-ERC-239993).

\section*{Appendix: Proofs of the propositions}

We first prove a lemma which will be used for the propositions.
\begin{lemma}
\label{lemma:bound}
Assume that the nonnegative sequence $\{u_k\}$ satisfies the following recursion for all $k \geqslant 1$:
$$
u_k^2 \leqslant S_k + \sum_{i=1}^k \lambda_i u_i,
$$
with $\{S_k\}$ an increasing sequence, $S_0 \geqslant u_0^2$ and $\lambda_i \geqslant 0$ for all $i$. Then, for all $k \geqslant 1$, then
$$
u_k \leqslant \frac{1}{2} \sum_{i=1}^k \lambda_i + \left(
S_k + \left( \frac{1}{2} \sum_{i=1}^k \lambda_i \right)^2
\right)^{1/2}
$$
\end{lemma}
\begin{proof}
We prove the result by induction. It is true for $k=0$ (by assumption). We assume it is true for  $k-1$, and we denote by $v_{k-1} = \max \{u_1,\dots,u_{k-1} \}  $. From the recursion, we thus get
$$
(u_k - \lambda_k / 2)^2 \leqslant S_k + \frac{\lambda_k^2}{4} +  v_{k-1} \sum_{i=1}^{k-1} \lambda_i
$$
leading to
$$
u_k  \leqslant  \frac{\lambda_k}{2}  + \left( S_k + \frac{\lambda_k^2}{4} +  v_{k-1} \sum_{i=1}^{k-1} \lambda_i
\right)^{1/2}
$$
and thus
$$
v_{k} \leqslant \max \left\{v_{k-1}, \frac{\lambda_k}{2}  + \left( S_k + \frac{\lambda_k^2}{4} +  v_{k-1} \sum_{i=1}^{k-1} \lambda_i
\right)^{1/2} \right\}
$$
The two terms in the maximum are equal if $v_{k-1}^2 = S_k + v_{k-1} \sum_{i=1}^k \lambda_i$, i.e., for
$v_{k-1}^\ast = \frac{1}{2} \sum_{i=1}^k \lambda_i + \left(
S_k + \left( \frac{1}{2} \sum_{i=1}^k \lambda_i \right)^2
\right)^{1/2}$. If $v_{k-1} \leqslant v_{k-1}^\ast$, then $v_{k} \leqslant v_{k-1}^\ast$ since the two terms in the $\max$ are increasing functions of $v_{k-1}$. If
$v_{k-1} \geqslant v_{k-1}^\ast$, then $v_{k-1} \geqslant \frac{\lambda_k}{2}  + \left( S_k + \frac{\lambda_k^2}{4} +  v_{k-1} \sum_{i=1}^{k-1} \lambda_i
\right)^{1/2}$. Hence, $v_{k} \leqslant v_{k-1}$, and the induction hypotheses ensure that the property is satisfied for $k$.
\end{proof}

The following lemma will allow us to characterize the elements of the $\varepsilon_k$-subdifferential of h at $x_k$, $\partial_{  \varepsilon_k} h  (x_k)$. As a reminder, the $\varepsilon$-subdifferential of a convex function $a$ at $x$ is the set of vectors $y$ such that $a(t) - a(x) \geqslant y^\top (t - x) - \varepsilon$ for all $t$.
\begin{lemma}
\label{lemma:subdifferential}
If $x_i$ is an $\ve_i$-optimal solution to the proximal problem~\eqref{eq:prox} in the sense of~\eqref{eq:3}, then there exists $f_i$ such that $\| f_i \| \leqslant \sqrt{\frac{2\varepsilon_i}{L}}$ and
\[
L\left(y_{i-1} - x_i -    \frac{1}{L} ( g'(y_{i-1}) + e_{i}) - f_i\right)  \in \partial_{  \varepsilon_i} h  (x_i) \; .
\]
\end{lemma}
\begin{proof}
We first recall some properties of $\varepsilon$-subdifferentials (see, e.g.,~\cite[Section 4.3]{bertsekas2003convex} for more details).
By definition, $x$ is an   $\varepsilon$-minimizer  of a convex function $a$ if and only if
$
 a(x) \leqslant \inf_{y \in \rb^n} a(y)~+~\varepsilon
$.
This is equivalent to $0$ belonging to the $\varepsilon$-subdifferential $\partial_\varepsilon a(x)$. If $a = a_1+a_2$, where both $a_1$ and $a_2$ are convex, we have
$
\partial_\varepsilon a(x) \subset \partial_\varepsilon a_1(x) + \partial_\varepsilon a_2(x)
$.

If $a_1(x) = \frac{L}{2} \| x - z \|^2$, then 
\begin{align*}
\partial_\varepsilon a_1(x) &= \left\{ y \in \rb^n \left| \ \frac{L}{2}\left\| x - z - \frac{y}{L} \right\|^2 \leqslant \varepsilon \right.\right\}\\
& = \left\{ y \in \rb^n, y = Lx - Lz + Lf \left|  \ \frac{L}{2}\| f \|^2 \leqslant \varepsilon \right.\right\}\; .
\end{align*}

If $a_2 = h$ and $x$ is an $\varepsilon$-minimizer of $a_1 + a_2$, then $0$ belongs to $\partial_\varepsilon a(x)$. Since $\partial_\varepsilon a(x) \subset \partial_\varepsilon a_1(x) + \partial_\varepsilon a_2(x)$, we have that $0$ is the sum of an element of $\partial_\varepsilon a_1(x)$ and of an element of $\partial_\varepsilon h(x)$. Hence, there is an $f$ such that
\begin{equation}
 Lz - Lx - Lf \in \partial_\varepsilon h (x)   \mbox{ with } \| f\| \leqslant \sqrt{\frac{2\varepsilon}{L}}.
\label{eq:f_definition}
\end{equation}
Using $z = y_{i-1}- (1/L) (g'(y_{i-1}) + e_i)$ and $x = x_i$, this implies that there exists $f_i$ such that $\| f_i \| \leqslant \sqrt{\frac{2\varepsilon_i}{L}}$ and
$$
L\left(y_{i-1} - x_{i} -    \frac{1}{L} ( g'(y_{i-1}) + e_{i}) - f_i\right)  \in \partial_{  \varepsilon_i} h  (x_i) \; .
$$
\end{proof}

In~\cite[Definition 2.1]{villa2011accelerated}, Eq.~(\ref{eq:f_definition}) is replaced by $Lz - Lx \in \partial_\varepsilon h (x)$. Hence, their definition of an approximate solution is equivalent to ours but using $f = 0$. If we replace $\|f_i\|$ by 0 in the proof of Proposition~\ref{prop:accelerated_convex}, we get the $O(1/k^2)$ convergence rate using any sequence of errors $\{\varepsilon_k\}$ necessary to achieve the $O(1/k^2)$ rate in~\cite[Th.~4.4]{villa2011accelerated}. We can also make the same assumption on $f$ in Proposition~\ref{prop:convex} to achieve the optimal convergence rate with a decay of $\sqrt{\varepsilon_k}$ in $O\left(1/k^{0.5+\delta}\right)$ instead of $O\left(1/k^{1+\delta}\right)$.

\subsection{Basic proximal-gradient method with errors in the convex case}

We now give the proof of Proposition of~\ref{prop:convex}.
\begin{proof}
Since $x_k$ is an $\ve_k$-optimal solution to the proximal problem~\eqref{eq:prox} in the sense of~\eqref{eq:3}, we can use Lemma~\ref{lemma:subdifferential} to yield that there exists $f_k$ such that $\| f_k \| \leqslant \sqrt{\frac{2\varepsilon_k}{L}}$ and
$$
L\left(x_{k-1} - x_{k} -    \frac{1}{L} ( g'(x_{k-1}) + e_{k}) - f_k\right)  \in \partial_{  \varepsilon_i} h  (x_k) \; .
$$

We now bound $g(x_i)$ and $h(x_i)$ as follows:
\begin{align*}
g(x_{i})    & \leqslant g(x_{i-1}) + \left\langle g'(x_{i-1}),x_{i}-x_{i-1}\right\rangle + \frac{L}{2}\|x_{i}-x_{i-1}\|^2\\
            & \hspace*{.3cm}\mbox{ using L-Lipschitz gradient and the convexity of $g$,} \\
            & \leqslant g(x^*) + \left\langle g'(x_{i-1}), x_{i-1}-x^*\right\rangle  + \left\langle g'(x_{i-1}),x_{i}-x_{i-1}\right\rangle + \frac{L}{2}\|x_{i}-x_{i-1}\|^2\\
            & \hspace*{.3cm}\mbox{ using convexity of $g$.}
\end{align*}

Using the $\varepsilon_i$-subgradient, we have
\[
h(x_i) \leqslant h(x^*) - \left\langle  g'(x_{i-1})  + e_i + L ( x_{i} + f_i - x_{i-1} ) , x_{i} - x^*\right\rangle + \varepsilon_i \; .
\]

Adding the two together, we get:
\begin{align*}
f(x_{i}) & = g(x_{i}) + h(x_{i})\\
 & =  f(x^*)  + \frac{L}{2}\|x_{i}-x_{i-1}\|^2   - L \left\langle    x_{i} - x_{i-1} , x_{i} - x^*\right\rangle + \varepsilon_i
 -    \left\langle  e_i+ L f_i,  x_{i} - x^*\right\rangle   \\
 & =  f(x^*)  + \frac{L}{2}\left\langle x_{i}-x_{i-1}, x_{i}-x_{i-1} - 2   x_{i} + 2x^*\right\rangle + \varepsilon_i
 -    \left\langle  e_i + L f_i,  x_{i} - x^*\right\rangle   \\
  & =  f(x^*)  + \frac{L}{2}\left\langle x_{i}-x^\ast - (x_{i-1} - x^\ast), (x^\ast - x_i) + (x^\ast - x_{i-1})\right\rangle + \varepsilon_i
 -    \left\langle  e_i + L f_i,  x_{i} - x^*\right\rangle   \\
  & =  f(x^*)  - \frac{L}{2}\| x_{i}-x^\ast\|^2 +  \frac{L}{2}\| x_{i-1}-x^\ast\|^2 + \varepsilon_i
 -    \left\langle  e_i + L f_i,  x_{i} - x^*\right\rangle   \\
f(x_i) & \leqslant  f(x^*)  - \frac{L}{2}\| x_{i}-x^\ast\|^2 +  \frac{L}{2}\| x_{i-1}-x^\ast\|^2 + \varepsilon_i
 + (\|  e_i \|  + \sqrt{2L\varepsilon_i} ) \cdot \|  x_{i} - x^*\|\\
 & \hspace*{.3cm}\mbox{using Cauchy-Schwartz and $\|f_i\| \leqslant \sqrt{\frac{2\varepsilon_i}{L}}$.}
\end{align*}
Moving $f(x^\ast)$ on the other side and summing from $i=1$ to $k$, we get:
$$
\sum_{i=1}^k [f(x_{i}) - f(x^\ast)] \leqslant -\frac{L}{2}\| x_{k}-x^\ast\|^2 + \frac{L}{2}\| x_{0}-x^\ast\|^2 + \sum_{i=1}^{k} \varepsilon_i + \sum_{i=1}^{k}\left[(\|  e_i \| + \sqrt{2L\varepsilon_i}) \cdot \|  x_{i} - x^*\| \right] \; ,
$$
i.e.
\begin{align}
\sum_{i=1}^k [f(x_{i}) - f(x^\ast)] + \frac{L}{2}\| x_{k}-x^\ast\|^2 &\leqslant \frac{L}{2}\| x_{0}-x^\ast\|^2 + \sum_{i=1}^{k} \varepsilon_i + \sum_{i=1}^{k}\left[(\|  e_i \| + \sqrt{2L\varepsilon_i}) \cdot \|  x_{i} - x^*\| \right] \; .
\label{eq:bound_convex_normal_1}
\end{align}
Eq.~(\ref{eq:bound_convex_normal_1}) has two purposes. The first one is to bound the values of $\|  x_{i} - x^*\|$ using the recursive definition. Once we have a bound on these quantities, we shall be able to bound the function values using only $\| x_{0}-x^\ast\|$ and the values of the errors.

\subsubsection{Bounding $\|  x_{i} - x^*\|$}

We now need to bound the quantities $\|  x_{i} - x^*\|$ in terms of $\| x_{0}-x^\ast\|$, $e_i$ and $\varepsilon_i$. Dropping the first term in Eq.~(\ref{eq:bound_convex_normal_1}), which is positive due to the optimality of $f(x^*)$, we have:
\[
    \| x_{k}-x^\ast\|^2 \leqslant \| x_{0}-x^\ast\|^2 + \frac{2}{L}\sum_{i=1}^{k} \varepsilon_i + 2\sum_{i=1}^{k}\left[\left(\frac{\|  e_i \|}{L} + \sqrt{\frac{2\varepsilon_i}{L}}\right) \cdot \|  x_{i} - x^*\| \right]
\]

We now use Lemma~\ref{lemma:bound} (using $S_k = \| x_{0}-x^\ast\|^2 + \frac{2}{L}\sum_{i=1}^{k} \varepsilon_i$ and $\lambda_i = 2\left[\frac{\|  e_i \|}{L} + \sqrt{\frac{2\varepsilon_i}{L}}\right]$) to get
$$
 \| x_{k}-x^\ast\|  \leqslant
 \sum_{i=1}^k \left(\frac{\|  e_i \|}{L} +   \sqrt{\frac{2\varepsilon_i}{L}} \right) + \left(
  \| x_{0}-x^\ast\|^2 + \frac{2}{L} \sum_{i=1}^{k}
\varepsilon_i  + \left[  \sum_{i=1}^k \left(\frac{\|  e_i \|}{L} +   \sqrt{\frac{2\varepsilon_i}{L}} \right)   \right]^2
\right)^{1/2} \; .
$$

Denoting $A_k = \sum_{i=1}^k \left(\frac{\|  e_i \|}{L} +   \sqrt{\frac{2\varepsilon_i}{L}} \right)$
and $B_k = \sum_{i=1}^k \frac{\varepsilon_i}{L}$, we get
$$
 \| x_{k}-x^\ast\|  \leqslant
A_k+ \left(
  \| x_{0}-x^\ast\|^2 + 2 B_k  + A_k^2
\right)^{1/2} \; .
$$
Since $A_i$ and $B_i$ are increasing sequences ($\|e_i\|$ and $\varepsilon_i$ being positive), we have for $i \leqslant k$
\begin{align*}
\| x_i-x^\ast\|  &\leqslant A_i+ \left( \| x_{0}-x^\ast\|^2 + 2 B_i  + A_i^2 \right)^{1/2}\\
                &\leqslant A_k+ \left( \| x_{0}-x^\ast\|^2 + 2 B_k  + A_k^2 \right)^{1/2}\\
                &\leqslant A_k+ \| x_{0}-x^\ast\| + \sqrt{2B_k}   + A_k\\
                &\hspace*{.3cm}\mbox{using the positivity of $\| x_{0}-x^\ast\|^2$, $B_k$ and $A_k^2$.}
\end{align*}

\subsubsection{Bounding the function values}
Now that we have a common bound for all $\| x_i-x^\ast\|$ with $i \leqslant k$, we can upper-bound the right-hand side of Eq.~(\ref{eq:bound_convex_normal_1}) using only terms depending on $\| x_{0}-x^\ast\|$, $e_i$ and $\varepsilon_i$.

Indeed, discarding $\frac{L}{2}\| x_{k}-x^\ast\|^2$ which is positive, Eq.~(\ref{eq:bound_convex_normal_1}) becomes
\BEAS
\sum_{i=1}^k [f(x_{i}) - f(x^\ast)] & \leqslant & \frac{L}{2} \| x_0 -x^\ast\|^2 + L B_k + L A_k (A_k+ \| x_{0}-x^\ast\| + \sqrt{2B_k}   + A_k)  \\
& \leqslant & \frac{L}{2} \| x_0 -x^\ast\|^2 + L B_k + 2L A_k^2 + LA_k     \| x_{0}-x^\ast\| + LA_k \sqrt{2B_k}  \\
& \leqslant & \frac{L}{2} \left(\| x_0 -x^\ast\| + 2A_k + \sqrt{2B_k}  \right) ^2 .
\EEAS
Since $f$ is convex, we get
\begin{align*}
f\left(\frac{1}{k}\sum_{i=1}^kx_{i}\right) - f(x^\ast) &\leqslant \frac1k\sum_{i=1}^k [f(x_{i}) - f(x^\ast)]\\
            &\leqslant \frac{L}{2k} \left(\| x_0 -x^\ast\| + 2A_k + \sqrt{2B_k}  \right) ^2 .
\end{align*}
\end{proof}

\subsection{Accelerated proximal-gradient method with errors in the convex case}

We now give the proof of Proposition~\ref{prop:accelerated_convex}.
\begin{proof}
Defining
\begin{align*}
\theta_k & = 2/(k+1)\\
v_k & = x_{k-1} + \frac{1}{\theta_k}(x_k - x_{k-1})\; ,
\end{align*}
we can rewrite the update for $y_k$ as
\[
y_k = (1-\theta_{k+1})x_k + \theta_{k+1}v_k \; ,
\]
because
\begin{align*}
(1-\theta_{k+1})x_k + \theta_{k+1}v_k & = (1-\frac{2}{k+2})x_k + \frac{2}{k+2}[x_{k-1} + \frac{k+1}{2}(x_k-x_{k-1})]\\
& = x_k - \frac{2}{k+2}(x_k - x_{k-1}) + \frac{k+1}{k+2}(x_k - x_{k-1})\\
& = x_k - \frac{k-1}{k+2}(x_k - x_{k-1}) = y_k.
\end{align*}
Because $g'$ is Lipschitz and $g$ is convex, we get for any $z$ that
\begin{align*}
g(x_k) & \leqslant g(y_{k-1}) + \left\langle g'(y_{k-1}),x_k-y_{k-1}\right\rangle + \frac{L}{2}\|x_k-y_{k-1}\|^2\\
& \leqslant g(z) + \left\langle g'(y_{k-1}),y_{k-1}-z\right\rangle + \left\langle g'(y_{k-1}),x_k-y_{k-1}\right\rangle + \frac{L}{2}\|x_k-y_{k-1}\|^2 \; .
\end{align*}
Because $- [ g'(y_{k-1})  + e_k + L ( x_{k} + f_k -  y_{k-1} ) ] \in \partial_{\varepsilon_k} h(x_{k})
$, we have for any $z$ that
\begin{align*}
h(x_k) & \leqslant \varepsilon_{k} + h(z) + \left\langle L(y_{k-1} - x_k) - g'(y_{k-1}) -e_k + L f_k,x_k-z\right\rangle\\
& = \varepsilon_{k} +  h(z) + \left\langle g'(y_{k-1}),z - x_k\right\rangle + L\left\langle x_k - y_{k-1},z - x_k \right\rangle
+ \left\langle e_k + L f_k, z - x_k \right\rangle
\end{align*}
Adding these bounds together gives:
\begin{align*}
g(x_k) + h(x_k) = f(x_k) \leqslant \varepsilon_{k} + f(z) + L\left\langle x_k - y_{k-1},z - x_k \right\rangle + \frac{L}{2}\|x_k-y_{k-1}\|^2 + \left\langle e_k + L f_k, z - x_k \right\rangle
\end{align*}
Choosing $z = \theta_kx^* + (1-\theta_k)x_{k-1}$ gives
\begin{align}
f(x_k) & \leqslant \varepsilon_{k} +  f(\theta_kx^* + (1-\theta)x_{k-1}) + L\left\langle x_k -y_{k-1},\theta_kx^* + (1-\theta_k)x_{k-1} - x_k\right\rangle + \frac{L}{2}\|x_k-y_{k-1}\|^2  \nonumber\\
 & \hspace*{1cm} +  \left\langle e_k + Lf_k, \theta_kx^* + (1-\theta_k)x_{k-1} - x_k \right\rangle\nonumber\\
& \leqslant  \varepsilon_{k} +  \theta_kf(x^*) + (1-\theta_k)f(x_{k-1}) + L\left\langle x_k -y_{k-1},\theta_kx^* + (1-\theta_k)x_{k-1} - x_k\right\rangle + \frac{L}{2}\|x_k-y_{k-1}\|^2\nonumber
\\
 & \hspace*{1cm} +  \left\langle e_k + Lf_k, \theta_kx^* + (1-\theta_k)x_{k-1} - x_k\right\rangle\label{eq:bound_convex_accelerated_1}\\
    &\hspace*{1cm}\mbox{using the convexity of $f$ and the fact that $\theta_k$ is in $[0,1]$.}\nonumber
\end{align}

Since
\[
\theta_kx^* + (1-\theta_k)x_{k-1} - x_k  = \theta_k (x^* - v_k)
\]
and
\begin{align*}
x_k - y_{k-1} & = \theta_k v_k  + (1 - \theta_k)x_{k-1} - y_{k-1}\\
			& = \theta_k v_k - \theta_k v_{k-1} \; ,
\end{align*}
we have
\begin{align}
L\left\langle x_k -y_{k-1},\theta_kx^* + (1-\theta_k)x_{k-1} - x_k\right\rangle & = L \theta_k^2 \left\langle v_k - v_{k-1}, x^* - v_k \right\rangle\nonumber\\
& =  - L \theta_k^2 \|v_k - x^*\|^2 + L \theta_k^2 \left\langle v_k - x^*, v_{k-1} - x^*\right\rangle\label{eq:bound_convex_accelerated_2}\\
\frac{L}{2}\|x_k-y_{k-1}\|^2 & = \frac{L\theta_k^2}{2} \|v_k - v_{k-1}\|^2\nonumber\\
						& = \frac{L\theta_k^2}{2} \left(\|v_k - x^*\|^2 + \|v_{k-1}-x^*\|^2 - 2 \left\langle v_k - x^*, v_{k-1} - x^*\right\rangle \right)\label{eq:bound_convex_accelerated_3}\\
\left\langle e_k + Lf_k, \theta_kx^* + (1-\theta_k)x_{k-1} - x_k \right\rangle & = \theta_k\left\langle e_k + Lf_k, x^* - v_k \right\rangle \; .\nonumber
\end{align}
Summing Eq.~(\ref{eq:bound_convex_accelerated_2}) and~(\ref{eq:bound_convex_accelerated_3}), we get
\[
L\left\langle x_k -y_{k-1},\theta_kx^* + (1-\theta_k)x_{k-1} - x_k\right\rangle + \frac{L}{2}\|x_k-y_{k-1}\|^2  =\frac{L \theta_k^2}{2} \left( \|v_{k-1} - x^*\|^2 -  \|v_k - x^*\|^2\right)
\]

Moving all function values in Eq.~(\ref{eq:bound_convex_accelerated_1}) to the left-side, we then get
\begin{align*}
f(x_k) - \theta_kf(x^*) -  (1-\theta_k)f(x_{k-1}) 
 &\leqslant L \theta_k^2 \left( \|v_{k-1} - x^*\|^2 -  \|v_k - x^*\|^2\right) +  \varepsilon_{k} +  \theta_k \left\langle e_k + L f_k, x^* - v_k \right\rangle \; .
\end{align*}
Reordering the terms and dividing by $\theta_k^2$ gives
\[
\frac{1}{\theta_k^2}(f(x_k)-f(x^*)) + \frac{L}{2}\|v_k - x^*\|^2 \leqslant
 \frac{1-\theta_k}{\theta_k^2}(f(x_{k-1})-f(x^*)) + \frac{L}{2}\|v_{k-1}-x^*\|^2
+ \frac{\varepsilon_{k} }{\theta_k^2} + \frac{1}{\theta_k} \left\langle e_k + L f_k, x^* - v_k \right\rangle.
\]
Now we use that for all $k$ greater than or equal to 1,
\[
\frac{1-\theta_k}{\theta_k^2} \leqslant \frac{1}{\theta_{k-1}^2}
\]
to apply this recursively and obtain
\begin{align*}
\frac{1}{\theta_k^2}(f(x_k)-f(x^*)) + \frac{L}{2}\|v_k - x^*\|^2 &\leqslant \frac{1-\theta_0}{\theta_0^2}(f(x_0)-f(x^*)) + \frac{L}{2}\|v_0-x^*\|^2
+ \sum_{i=1}^{k} \frac{\varepsilon_i }{\theta_i^2}\\
&\hspace*{1cm}+ \sum_{i=1}^{k}  \frac{1}{\theta_i} (\| e_i \| + \sqrt{2L\varepsilon_i})  \cdot\| x^* - v_i \|
\end{align*}
using $\|f_i\| \leqslant \sqrt{\frac{2\varepsilon_i}{L}}$. Since $v_0 = x_0$ and $\theta_0=2$, we get
\beq
f(x_k) - f(x^*)  + \frac{L \theta_k^2}{2}\|v_k - x^*\|^2\leqslant \frac{L\theta_k^2}{2}\|x_0-x^*\|^2
+  \theta_k^2 \sum_{i=1}^{k} \frac{\varepsilon_i }{\theta_i^2}
+ \theta_k^2 \sum_{i=1}^{k}  \frac{1}{\theta_i} \left(\| e_i \| + \sqrt{2L\varepsilon_i}\right) \cdot\| x^* - v_i \| \; .\label{eq:bound_convex_accelerated_4}
\eeq

As in the previous proof, we will now use Eq.~(\ref{eq:bound_convex_accelerated_4}) to first bound the values of $\| v_i - x^*\|$ then, using these bounds, bound the function values.

\subsubsection{Bounding $\| v_i - x^*\|$}

We now need to bound the quantities $\|v_i - x^*\|$ in terms of $\|x_0-x^*\|$, $e_i$ and $\ve_i$.
\[
\|v_k - x^*\|^2\leqslant \|x_0-x^*\|^2 + \frac{2}{L} \sum_{i=1}^{k} \frac{\varepsilon_i }{\theta_i^2} + \sum_{i=1}^{k}  \frac{2}{\theta_i} \left(\frac{\| e_i \|}{L} + \sqrt{\frac{2\varepsilon_i}{L}}\right) \cdot\| x^* - v_i \| \; .
\]
Since $\theta_i = 2/(i+1)$, $\frac{1}{\theta_i} = \frac{i+1}{2} \leqslant i$ since $i \geqslant 1$. Thus, we have
\[
\|v_k - x^*\|^2\leqslant \|x_0-x^*\|^2 + \frac{2}{L} \sum_{i=1}^{k} i^2\varepsilon_i + \sum_{i=1}^{k}  2i\left(\frac{\| e_i \|}{L} + \sqrt{\frac{2\varepsilon_i}{L}}\right) \cdot\| x^* - v_i \| \; .
\]

From Lemma~\ref{lemma:bound} (using $S_k = \| x_{0}-x^\ast\|^2 + \frac{2}{L}\sum_{i=1}^{k} i^2\varepsilon_i$ and $\lambda_i = 2i \left[\frac{\| e_i \|}{L} + \sqrt{\frac{2\varepsilon_i}{L}}\right]$), and denoting $\widetilde{A}_k = \sum_{i=1}^k i\left(\frac{\|  e_i \|}{L} +   \sqrt{\frac{2\varepsilon_i}{L}} \right) $
and $\widetilde{B}_k = \sum_{i=1}^k \frac{i^2\varepsilon_i }{L}$, we get
\[
 \| v_{k}-x^\ast\|  \leqslant \widetilde{A}_k+ \left( \| x_{0}-x^\ast\|^2 + 2 \widetilde{B}_k  + \widetilde{A}_k^2
\right)^{1/2}\; .
\]
Since $\widetilde{A}_i$ and $\widetilde{B}_i$ are increasing sequences, we also have for $i \leqslant k$:
\begin{align*}
\| v_{i}-x^\ast\|  	&\leqslant \widetilde{A}_i+ \left( \| x_{0}-x^\ast\|^2 + 2 \widetilde{B}_i  + \widetilde{A}_i^2
\right)^{1/2}\\
				&\leqslant \| x_{0}-x^\ast\| + 2\widetilde{A}_i+ \widetilde{B}_i^{1/2}\sqrt{2}\\
				&\leqslant \| x_{0}-x^\ast\| + 2\widetilde{A}_k+ \widetilde{B}_k^{1/2}\sqrt{2} \; .
\end{align*}

\subsubsection{Bounding the function values}

Dropping $\frac{L \theta_k^2}{2}\|v_k - x^*\|^2$ in Eq.~(\ref{eq:bound_convex_accelerated_4}) (since it is positive), we thus have
\begin{align*}
f(x_k) - f(x^*) & \leqslant \frac{L\theta_k^2}{2}\left(\|x_0-x^*\|^2 + 2\widetilde{B}_k + 2\widetilde{A}_k\left[\| x_{0}-x^\ast\| + 2\widetilde{A}_k+ \sqrt{2\widetilde{B}_k}\right]\right)\\
& \leqslant \frac{L\theta_k^2}{2}\left(\|x_0-x^*\|^2 + 2\widetilde{B}_k + 2\widetilde{A}_k\| x_{0}-x^\ast\| + 4\widetilde{A}_k^2+ 2\widetilde{A}_k\sqrt{2\widetilde{B}_k}\right)\\
			& \leqslant \frac{L\theta_k^2}{2}\left(\|x_0-x^*\| + 2\widetilde{A}_k +\sqrt{2\widetilde{B}_k} \right)^2
\end{align*}

 and
 $$
 \frac{1}{\theta_k^2}(f(x_k)-f(x^*)) \leqslant
 \frac{L}{2}
\left(
\| x_0 -x^\ast\| + 2\widetilde{A}_k + \sqrt{2 \widetilde{B}_k}  \right) ^2 .
$$

\end{proof}

\subsection{Basic proximal-gradient method with errors in the strongly convex case}

Below is the proof of Proposition~\ref{prop:strongly_convex}

\begin{proof}
Again, there exists $f_i$ such that $\| f_i \| \leqslant \sqrt{\frac{2\varepsilon_i}{L}}$ and
$$
L\left(x_{i-1} - x_{i} -    \frac{1}{L} ( g'(x_{i-1}) + e_{i}) - f_i\right)  \in \partial_{  \varepsilon_i} h  (x_i) \; .
$$

Since $x^\ast$ is optimal, we have that $x^* = \prox_L\left(x^* - \frac{1}{L}g'(x^*)\right) \; .$

We first separate $f_k$, the error in the proximal, from the rest:
\begin{align*}
\|x_k - x^*\|^2 & = \left\|\prox_L\left(x_{k-1} - \frac{1}{L}g'(x_{k-1}) - \frac{1}{L}e_k\right)
+ f_k - \prox_L\left(x^* -
\frac{1}{L}g'(x^*)\right)\right\|^2\\
    & = \left\|\prox_L\left(x_{k-1} - \frac{1}{L}g'(x_{k-1}) - \frac{1}{L}e_k\right) - \prox_L\left(x^* -
\frac{1}{L}g'(x^*)\right)\right\|^2 + \|f_k\|^2\\
    &\hspace*{.3cm} + 2\left\langle f_k, \prox_L\left(x_{k-1} - \frac{1}{L}g'(x_{k-1}) -
\frac{1}{L}e_k\right) - \prox_L\left(x^* -
\frac{1}{L}g'(x^*)\right) \right\rangle\\
    & \leqslant \left\|\prox_L\left(x_{k-1} - \frac{1}{L}g'(x_{k-1}) - \frac{1}{L}e_k\right) - \prox_L\left(x^* -
\frac{1}{L}g'(x^*)\right)\right\|^2 + \frac{2\ve_k}{L}\\
    &\hspace*{.3cm} + 2\sqrt{\frac{2\ve_k}{L}}\left\|\prox_L\left(x_{k-1} - \frac{1}{L}g'(x_{k-1}) -
\frac{1}{L}e_k\right) - \prox_L\left(x^* -
\frac{1}{L}g'(x^*)\right)\right\|\\
    &\hspace*{.3cm}\mbox{using Cauchy-Schwartz and $\|f_k\| \leqslant \sqrt{\frac{2\ve_k}{L}}$}\\
    & \leqslant \left\|x_{k-1} - \frac{1}{L}g'(x_{k-1}) - \frac{1}{L}e_k - x^* +
\frac{1}{L}g'(x^*)\right\|^2 +  \frac{2\ve_k}{L}\\
    &\hspace*{.3cm} + 2\sqrt{\frac{2\ve_k}{L}}\left\|x_{k-1} - \frac{1}{L}g'(x_{k-1}) -
\frac{1}{L}e_k - x^* +\frac{1}{L} g'(x^*)\right\|\\
    &\hspace*{.3cm}\mbox{using the non-expansiveness of the
proximal}\\
    & \leqslant \left\|x_{k-1} - \frac{1}{L}g'(x_{k-1}) - \frac{1}{L}e_k - x^* +
\frac{1}{L}g'(x^*)\right\|^2 +  \frac{2\ve_k}{L}\\
    &\hspace*{.3cm} + 2\sqrt{\frac{2\ve_k}{L}}\left(\left\|x_{k-1} - x^* - \frac{1}{L}(g'(x_{k-1})-
g'(x^*))\right\| + \frac{\|e_k\|}{L}\right)\\
	&\hspace*{.3cm}\mbox{using the triangular inequality.}
\end{align*}
We continue this computation, but now separating $e_k$, the error in the gradient, from the rest:
\begin{align*}
 \|x_k - x^*\|^2   & = \|x_{k-1} - x^* - \frac{1}{L}(g'(x_{k-1})-
g'(x^*))\|^2 + \frac{\|e_k\|^2}{L^2}  - \frac{2}{L} \left\langle
e_k, x_{k-1} - x^* - \frac{1}{L}(g'(x_{k-1}) - \frac{1}{L}g'(x^*)) \right\rangle\\
    &\hspace*{.3cm}  +  \frac{2\ve_k}{L} + 2 \sqrt{\frac{2\ve_k}{L}} \left(\|x_{k-1} - x^* - \frac{1}{L}(g'(x_{k-1})
- g'(x^*))\| + \frac{\|e_k\|}{L}\right)\\
    & \leqslant \|x_{k-1} - x^* - \frac{1}{L}(g'(x_{k-1})-
g'(x^*))\|^2 + \frac{\|e_k\|^2}{L^2}  + \frac{2}{L} \|e_k\|
\|x_{k-1} - x^* - \frac{1}{L}(g'(x_{k-1}) - g'(x^*))\|\\
    &\hspace*{.3cm}  +  \frac{2\ve_k}{L} + 2 \sqrt{\frac{2\ve_k}{L}} \left(\|x_{k-1} - x^* - \frac{1}{L}(g'(x_{k-1})
-g'(x^*))\| + \frac{\|e_k\|}{L}\right)\\
    &\hspace*{.3cm}\mbox{using Cauchy-Schwartz}\\
    &\leqslant \|x_{k-1} - x^* - \frac{1}{L}(g'(x_{k-1})-
g'(x^*))\|^2 + \frac{\|e_k\|^2}{L^2}   +  \frac{2\ve_k}{L} +
\frac{2}{L}\sqrt{\frac{2\ve_k}{L}}\|e_k\| \\
    & + \left(\frac{2\|e_k\|}{L} + 2 \sqrt{\frac{2\ve_k}{L}}\right)\left\|x_{k-1} - x^* -
\frac{1}{L}(g'(x_{k-1}) - g'(x^*))\right\| \; .
\end{align*}
We now need to bound $\left\|x_{k-1} - x^* -
\frac{1}{L}(g'(x_{k-1}) - g'(x^*))\right\|$ to get the final result.
We have:
\begin{align*}
\|x_{k-1} - x^* - \frac{1}{L}(g'(x_{k-1}) -g'(x^*))\|^2 &=
\|x_{k-1} - x^*\|^2 + \frac{1}{L^2}\|g'(x_{k-1}) - g'(x^*)\|^2\\
&\hspace*{.3cm}- \frac{2}{L}\left\langle g'(x_{k-1}) - g'(x^*),x_{k-1}-x^*\right\rangle  \\
& \leqslant \|x_{k-1} - x^*\|^2 + \frac{1}{L^2}\|g'(x_{k-1}) -
g'(x^*)\|^2\\
&\hspace*{.3cm}- \frac{2}{L}\left(
\frac{1}{L+\mu}\|g'(x_{k-1})-g'(x^*)\|^2 +
\frac{L\mu}{L+\mu}\|x_{k-1}-x^*\|^2\right)\\
&\hspace*{.3cm}\mbox{using theorem 2.1.12 of~\cite{nesterov2004introductory}}\\
& = (1 - \frac{2\mu}{L+\mu})\|x_{k-1}-x^*\|^2 +
\frac{1}{L}\left(\frac{1}{L} - \frac{2}{L+\mu}\right)\|g'(x_{k-1}) -
g'(x^*)\|^2\\
&\leqslant (1 - \frac{2\mu}{L+\mu})\|x_{k-1}-x^*\|^2 +
\frac{\mu^2}{L}(\frac{1}{L}-\frac{2}{L+\mu}) \|x_{k-1}-x^*\|^2\\
&\hspace*{.3cm}\mbox{using the negativity of $\frac{1}{L} - \frac{2}{L + \mu}$ and the strong convexity of $g$}\\
& = \left(1 - \frac{\mu}{L}\right)^2\|x_{k-1}-x^*\|^2.
\end{align*}

Thus
\begin{eqnarray*}
 \|x_k - x^*\|^2   & \leqslant& \left(1 - \frac{\mu}{L}\right)^2 \|x_{k-1} -
x^*\|^2 + \frac{\|e_k\|^2}{L^2} +  \frac{2\ve_k}{L} + \frac{2}{L}\sqrt{\frac{2\ve_k}{L}}\|e_k\|\\
  &&\hspace*{.3cm}+ \left(\frac{2\|e_k\|}{L} +
2\sqrt{\frac{2\ve_k}{L}}\right)\left(1 - \frac{\mu}{L}\right)\|x_{k-1} - x^*\|\\
    & = &\left[\left(1 - \frac{\mu}{L}\right)\|x_{k-1} - x^*\| + \frac{\|e_k\|}{L} +
\sqrt{\frac{2\ve_k}{L}}\right]^2 \; .
\end{eqnarray*}
Taking the square root of both sides and applying the bound recursively yields
\[
	\|x_k - x^*\| \leqslant \left(1 - \frac{\mu}{L}\right)^k\|x_0 - x^*\| +
\sum_{i=1}^k \left(1 - \frac{\mu}{L}\right)^{k-i}
\left(\frac{\|e_i\|}{L} + \sqrt{\frac{2\ve_i}{L}}\right) \; .
\]

\end{proof}

\subsection{Accelerated proximal-gradient method with errors in the strongly convex case}
We now give the proof of Proposition~\ref{prop:accelerated_strongly_convex}.

\begin{proof}
We have (following~\cite{nesterov2004introductory})
\[
x_k = y_{k-1} - \frac{1}{L} g'(y_{k-1}) \; .
\]
We define
\begin{align*}
\alpha_k^2 &= (1 - \alpha_k)\alpha_{k-1}^2 + \frac{\mu}{L}\alpha_k\\
v_k & = x_{k-1} + \frac{1}{\alpha_{k-1}}(x_k - x_{k-1})\\
\theta_k & = \frac{\alpha_k - \frac{\mu}{L}}{1 - \frac{\mu}{L}}\\
y_k & = x_k + \theta_k (v_k - x_k) \; .
\end{align*}
If we choose $\alpha_0 = \sqrt{\gamma}$, then this yields
\begin{align*}
	y_k &= x_k + \frac{1-\sqrt{\gamma}}{1+\sqrt{\gamma}}(x_k - x_{k-1}) \; .
\end{align*}

We can bound $g(x_k)$ with
\begin{align*}
g(x_k) 	& \leqslant g(y_{k-1}) + \langle g'(y_{k-1}), x_k - y_{k-1} \rangle + \frac{L}{2} \|x_k - y_{k-1}\|^2\\
		& \hspace*{3cm} \mbox{using the convexity of $g$}\\
		& \leqslant g(z) + \langle g'(y_{k-1}), y_{k-1} - z\rangle + \langle g'(y_{k-1}), x_k - y_{k-1} \rangle + \frac{L}{2} \|x_k - y_{k-1}\|^2 - \frac{\mu}{2} \|y_{k-1} - z\|^2\\
		& \hspace*{3cm} \mbox{using the $\mu$-strong convexity of $g$.}
\end{align*}

Using Lemma~\ref{lemma:subdifferential}, we have that $- [ g'(y_{k-1})  + e_k + L ( x_{k} + f_k -  y_{k-1} ) ] \in \partial_{\varepsilon_k} h(x_{k})$. Hence, we have for any $z$ that
\begin{align*}
h(x_k) & \leqslant \varepsilon_k + h(z) + \left\langle L(y_{k-1} - x_k) - g'(y_{k-1})-e_k-Lf_k,x_k-z\right\rangle\\
& = \varepsilon_k + h(z) + \left\langle g'(y_{k-1}),z - x_k\right\rangle + L\left\langle x_k - y_{k-1},z - x_k \right\rangle+ \left\langle e_k+Lf_k,z - x_k\right\rangle
\end{align*}

Adding these two bounds, we get for any $z$
\begin{align*}
f(x_k) &\leqslant \varepsilon_k + f(z) +  L\left\langle x_k - y_{k-1},z - x_k \right\rangle + \frac{L}{2} \|x_k - y_{k-1}\|^2 - \frac{\mu}{2} \|y_{k-1} - z\|^2+ \left\langle e_k+Lf_k,z - x_k\right\rangle \; .\\
\end{align*}

Using $z = \alpha_{k-1}x^* + (1 - \alpha_{k-1})x_{k-1}$, we get
\begin{align*}
f(x_k) 	&\leqslant  \varepsilon_k + f(\alpha_{k-1}x^* + (1 - \alpha_{k-1})x_{k-1}) + L\left\langle x_k - y_{k-1},\alpha_{k-1}x^* + (1 - \alpha_{k-1})x_{k-1} - x_k \right\rangle\\
		& \hspace*{.6cm} + \frac{L}{2} \|x_k - y_{k-1}\|^2 - \frac{\mu}{2} \|y_{k-1} - \alpha_{k-1}x^* - (1 - \alpha_{k-1})x_{k-1}\|^2\\
		&\hspace*{.6cm} + \left\langle e_k+Lf_k,\alpha_{k-1}x^* + (1 - \alpha_{k-1})x_{k-1} - x_k\right\rangle\\
		& \leqslant \varepsilon_k + \alpha_{k-1} f(x^*) + (1 - \alpha_{k-1}) f(x_{k-1}) + L\left\langle x_k - y_{k-1},\alpha_{k-1}x^* + (1 - \alpha_{k-1})x_{k-1} - x_k \right\rangle\\
		&  \hspace*{.6cm} + \frac{L}{2} \|x_k - y_{k-1}\|^2 - \frac{\mu}{2} \|y_{k-1} - \alpha_{k-1}x^* - (1 - \alpha_{k-1})x_{k-1}\|^2 - \frac{\mu}{2} \alpha_{k-1}(1 - \alpha_{k-1}) \|x^* - x_{k-1}\|^2\\
		&  \hspace*{.6cm}+ \left\langle e_k+Lf_k,\alpha_{k-1}x^*+ (1 - \alpha_{k-1})x_{k-1} - x_k\right\rangle\\
		&\mbox{using the $\mu$-strong convexity of $f$.}
\end{align*}
We can replace $x_k - y_{k-1}$ using
\begin{align*}
x_k - y_{k-1} 	& = x_k - x_{k-1} - \theta_{k-1} (v_{k-1} - x_{k-1})\\
			& = \theta_{k-1} x_{k-1} + \frac{\theta_{k-1}}{\alpha_{k-1}}(x_k - x_{k-1}) + \left(1 - \frac{\theta_{k-1}}{\alpha_{k-1}}\right)(x_k - x_{k-1}) - \theta_{k-1} v_{k-1}\\
			& = \theta_{k-1} (v_k - v_{k-1}) + \left(1 - \frac{\theta_{k-1}}{\alpha_{k-1}}\right)(x_k - x_{k-1}) \; .
\end{align*}
We also have
\begin{align*}
(1 - \alpha_{k-1})x_{k-1} - x_k	& = -\alpha_{k-1} v_k\\
\alpha_{k-1}x^* + (1 - \alpha_{k-1})x_{k-1} - x_k &= \alpha_{k-1} (x^* - v_k) \; ,
\end{align*}
and
\begin{align*}
y_{k-1} - \alpha_{k-1}x^* - (1 - \alpha_{k-1})x_{k-1}& = y_{k-1} - \alpha_{k-1} (x^* - v_k) - x_k\\
		& = \alpha_{k-1} (v_k - x^*) - \theta_{k-1} (v_k - v_{k-1}) - \left(1 - \frac{\theta_{k-1}}{\alpha_{k-1}}\right)(x_k - x_{k-1})
\end{align*}
Thus,
\begin{align*}
f(x_k) 	&  \leqslant \varepsilon_k + \alpha_{k-1} f(x^*) + (1 - \alpha_{k-1}) f(x_{k-1})\\
		& \hspace*{1cm}  - L \theta_{k-1}\alpha_{k-1} \langle v_k - v_{k-1}, v_k - x^*\rangle - L(\alpha_{k-1} - \theta_{k-1}) \langle x_k - x_{k-1}, v_k - x^*\rangle\\
		& \hspace*{1cm} + \frac{L\theta_{k-1}^2}{2}\|v_k - v_{k-1}\|^2 + \frac{L\left(1 - \frac{\theta_{k-1}}{\alpha_{k-1}}\right)^2}{2}\| x_k - x_{k-1}\|^2\\
		& \hspace*{1cm} + L \theta_{k-1} \left(1 - \frac{\theta_{k-1}}{\alpha_{k-1}}\right) \langle v_k - v_{k-1}, x_k - x_{k-1}\rangle\\
		& \hspace*{1cm} - \frac{\mu\alpha_{k-1}^2}{2}\| v_k - x^*\|^2 - \frac{\mu\theta_{k-1}^2}{2}\|v_k - v_{k-1}\|^2 - \frac{\mu\left(1 - \frac{\theta_{k-1}}{\alpha_{k-1}}\right)^2}{2} \|x_k - x_{k-1}\|^2\\
		& \hspace*{1cm} + \mu\alpha_{k-1}\theta_{k-1} \langle v_k - x^*, v_k - v_{k-1}\rangle + \mu(\alpha_{k-1} - \theta_{k-1}) \langle v_k - x^*, x_k - x_{k-1}\rangle\\
		& \hspace*{1cm}  - \mu\theta_{k-1}\left(1 - \frac{\theta_{k-1}}{\alpha_{k-1}}\right) \langle v_k - v_{k-1}, x_k - x_{k-1}\rangle - \frac{\mu}{2} \alpha_{k-1}(1 - \alpha_{k-1}) \|x^* - x_{k-1}\|^2\\
		 & \hspace*{1cm}+ \alpha_{k-1}\left\langle e_k+Lf_k,x^* - v_k\right\rangle \; .
\end{align*}
To avoid unnecessary clutter, we shall denote $E_k$ the additional term induced by the errors, i.e.
\[
E_k = \varepsilon_k +  \alpha_{k-1}\left\langle e_k+Lf_k,x^* - v_k\right\rangle \; .
\]
Before reordering the terms together, we shall also replace all instances of $v_k - v_{k-1}$ with $v_k - x^* - (v_{k-1} - x^*)$:
\begin{align*}
f(x_k) 	&  \leqslant E_k + \alpha_{k-1} f(x^*) + (1 - \alpha_{k-1}) f(x_{k-1})  - L \theta_{k-1}\alpha_{k-1} \|v_k - x^*\|^2 +L \theta_{k-1}\alpha_{k-1} \langle v_{k-1} - x^*, v_k - x^*\rangle\\
		& \hspace*{1cm}   - L(\alpha_{k-1} - \theta_{k-1}) \langle x_k - x_{k-1}, v_k - x^*\rangle\\
		& \hspace*{1cm} + \frac{(L - \mu)\theta_{k-1}^2}{2}  \|v_k - x^*\|^2 + \frac{(L - \mu)\theta_{k-1}^2}{2} \|v_{k-1} - x^*\|^2 - (L - \mu)\theta_{k-1}^2\langle v_k - x^*, v_{k-1} - x^* \rangle\\
		& \hspace*{1cm} + \frac{(L - \mu)\left(1 - \frac{\theta_{k-1}}{\alpha_{k-1}}\right)^2}{2}\| x_k - x_{k-1}\|^2\\
		& \hspace*{1cm} + L \theta_{k-1} \left(1 - \frac{\theta_{k-1}}{\alpha_{k-1}}\right) \langle v_k - x^*, x_k - x_{k-1}\rangle - L \theta_{k-1} \left(1 - \frac{\theta_{k-1}}{\alpha_{k-1}}\right) \langle v_{k-1} - x^*, x_k - x_{k-1}\rangle\\
		& \hspace*{1cm} - \frac{\mu\alpha_{k-1}^2}{2}\| v_k - x^*\|^2\\
		& \hspace*{1cm} + \mu\alpha_{k-1}\theta_{k-1} \|v_k - x^*\|^2 - \mu\alpha_{k-1}\theta_{k-1} \langle v_k - x^*, v_{k-1}- x^*\rangle\\
		&\hspace*{1cm} + \mu(\alpha_{k-1} - \theta_{k-1}) \langle v_k - x^*, x_k - x_{k-1}\rangle\\
		& \hspace*{1cm}  - \mu\theta_{k-1}\left(1 - \frac{\theta_{k-1}}{\alpha_{k-1}}\right) \langle v_k - x^*, x_k - x_{k-1}\rangle + \mu\theta_{k-1}\left(1 - \frac{\theta_{k-1}}{\alpha_{k-1}}\right) \langle v_{k-1} - x^*, x_k - x_{k-1}\rangle\\
		& \hspace*{1cm}  - \frac{\mu}{2} \alpha_{k-1}(1 - \alpha_{k-1}) \|x^* - x_{k-1}\|^2 \; .
\end{align*}
With a bit of well-needed cleaning, this becomes
\begin{align*}
f(x_k) 	&  \leqslant E_k + \alpha_{k-1} f(x^*) + (1 - \alpha_{k-1}) f(x_{k-1})\\
		& \hspace*{1cm}+ \left[\frac{L- \mu}{2}(\theta_{k-1} - \alpha_{k-1})^2 - \frac{L\alpha_{k-1}^2}{2}\right]\|v_k - x^*\|^2\\
		& \hspace*{1cm}+ \frac{(L - \mu)\theta_{k-1}^2}{2} \|v_{k-1} - x^*\|^2\\
		& \hspace*{1cm}+ (L-\mu) \theta_{k-1}(\alpha_{k-1}- \theta_{k-1})\langle v_{k-1} - x^*, v_k - x^*\rangle\\
		& \hspace*{1cm}+ (L-\mu) (\theta_{k-1} - \alpha_{k-1}) \left(1 - \frac{\theta_{k-1}}{\alpha_{k-1}}\right)\langle x_k - x_{k-1}, v_k - x^*\rangle\\
		& \hspace*{1cm}- (L - \mu)\theta_{k-1}\left(1 - \frac{\theta_{k-1}}{\alpha_{k-1}}\right)\langle v_{k-1} - x^*, x_k - x_{k-1}\rangle\\
		& \hspace*{1cm}+ \frac{(L - \mu)\left(1 - \frac{\theta_{k-1}}{\alpha_{k-1}}\right)^2}{2}\| x_k - x_{k-1}\|^2\\
		& \hspace*{1cm} - \frac{\mu}{2} \alpha_{k-1}(1 - \alpha_{k-1}) \|x^* - x_{k-1}\|^2 \; .
\end{align*}
We can rewrite $x_k - x_{k-1}$ using
\begin{align*}
x_k - x_{k-1} 	& = \alpha_{k-1} (v_k - x_{k-1})\\
			& = \alpha_{k-1}(v_k - x^*)-\alpha_{k-1}(x_{k-1} - x^*)\; .
\end{align*}
We may now compute the coefficients for the following terms: $\|v_k - x^*\|^2$, $\|v_{k-1} - x^*\|^2$, $\langle v_{k-1} - x^*, v_k - x^*\rangle$, $\langle x_{k-1} - x^*, v_k - x^*\rangle$, $\langle x_{k-1} - x^*, v_{k-1} - x^*\rangle$ and $\|x^* - x_{k-1}\|^2$.

For $\|v_k - x^*\|^2$, we have
\begin{align*}
\underbrace{\frac{L- \mu}{2}(\theta_{k-1} - \alpha_{k-1})^2 - \frac{L\alpha_{k-1}^2}{2}}_{\mbox{$\|v_k - x^*\|^2$ term}} - \underbrace{(L-\mu) (\theta_{k-1} - \alpha_{k-1})^2}_{\mbox{$\langle x_k - x_{k-1}, v_k - x^*\rangle$ term}} + \underbrace{\frac{(L - \mu)(\theta_{k-1} - \alpha_{k-1})^2}{2}}_{\mbox{$\| x_k - x_{k-1}\|^2$ term}} = - \frac{L\alpha_{k-1}^2}{2} \; .
\end{align*}

For $\|v_{k-1} - x^*\|^2$, there is only one term and we keep $\frac{(L - \mu)\theta_{k-1}^2}{2}$.

For $\langle v_{k-1} - x^*, v_k - x^*\rangle$, we get
\[
\underbrace{(L-\mu) \theta_{k-1}(\alpha_{k-1}- \theta_{k-1})}_{\mbox{$\langle v_{k-1} - x^*, v_k - x^*\rangle$ term}} - \underbrace{(L - \mu)\theta_{k-1}\left(\alpha_{k-1} - \theta_{k-1}\right)}_{\mbox{$\langle v_{k-1} - x^*, x_k - x_{k-1}\rangle$ term}} = 0 \; .
\]

For $\langle x_{k-1} - x^*, v_k - x^*\rangle$, we get
\[
\underbrace{(L-\mu) (\theta_{k-1} - \alpha_{k-1})^2}_{\mbox{$\langle x_k - x_{k-1}, v_k - x^*\rangle$ term}} - \underbrace{(L - \mu)(\theta_{k-1} - \alpha_{k-1})^2}_{\mbox{$\| x_k - x_{k-1}\|^2$ term}} = 0 \; .
\]

For $\langle x_{k-1} - x^*, v_{k-1} - x^*\rangle$, we get
\[
\underbrace{(L - \mu)\theta_{k-1}\left(\alpha_{k-1} - \theta_{k-1}\right)}_{\mbox{$\langle v_{k-1} - x^*, x_k - x_{k-1}\rangle$ term}} = (L - \mu)\theta_{k-1}\left(\alpha_{k-1} - \theta_{k-1}\right) \; .
\]

For $\|x^* - x_{k-1}\|^2$, we get
\[
\underbrace{ \frac{(L - \mu)(\theta_{k-1} - \alpha_{k-1})^2}{2}}_{\mbox{$\| x_k - x_{k-1}\|^2$ term}} - \underbrace{\frac{\mu}{2} \alpha_{k-1}(1 - \alpha_{k-1})}_{\mbox{$\|x^* - x_{k-1}\|^2$ term}} = - \frac{\mu}{2} \theta_{k-1}(1 - \alpha_{k-1}) \; .
\]

Hence, we have
\begin{align*}
f(x_k) 	&  \leqslant E_k + \alpha_{k-1} f(x^*) + (1 - \alpha_{k-1}) f(x_{k-1})\\
		& \hspace*{1cm}- \frac{L\alpha_{k-1}^2}{2}\|v_k - x^*\|^2\\
		& \hspace*{1cm}+ \frac{(L - \mu)\theta_{k-1}^2}{2} \|v_{k-1} - x^*\|^2\\
		& \hspace*{1cm}+ (L - \mu)\theta_{k-1}\left(\alpha_{k-1} - \theta_{k-1}\right)\langle v_{k-1} - x^*, x_{k-1} - x^*\rangle\\
		& \hspace*{1cm} - \frac{\mu}{2} \theta_{k-1}(1 - \alpha_{k-1}) \|x_{k-1} - x^*\|^2\\
		& \hspace*{1cm} - \frac{\theta_{k-1}(L - \mu)^2(\theta_{k-1} - \alpha_{k-1})^2 }{2\mu(1 -\alpha_{k-1})}\|v_{k-1} - x^*\|^2\\
		& \hspace*{1cm} + \frac{\theta_{k-1}(L - \mu)^2(\theta_{k-1} - \alpha_{k-1})^2 }{2\mu(1 -\alpha_{k-1})}\|v_{k-1} - x^*\|^2 \; ,
\end{align*}
the last two lines allowing us to complete the square. We may now factor it to get
\begin{align*}
f(x_k) 	&  \leqslant E_k + \alpha_{k-1} f(x^*) + (1 - \alpha_{k-1}) f(x_{k-1})\\
		& \hspace*{1cm}- \frac{L\alpha_{k-1}^2}{2}\|v_k - x^*\|^2\\
		& \hspace*{1cm}+ \frac{(L - \mu)\theta_{k-1}^2}{2} \|v_{k-1} - x^*\|^2\\
		& \hspace*{1cm} - \frac{\mu}{2} \theta_{k-1}(1 - \alpha_{k-1}) \left\|x_{k-1} - x^* - \frac{(L-\mu)(\alpha_{k-1} - \theta_{k-1})}{\mu(1 -\alpha_{k-1})}(v_{k-1} - x^*)\right\|^2\\
		& \hspace*{1cm} + \frac{\theta_{k-1}(L - \mu)^2(\theta_{k-1} - \alpha_{k-1})^2 }{2\mu(1 - \alpha_{k-1})}\|v_{k-1} - x^*\|^2 \; .
\end{align*}
Discarding the term depending on $x_{k-1} - x^*$ and regrouping the terms depending on $\|v_{k-1} - x^*\|^2$, we have
\begin{align*}
f(x_k) 	&  \leqslant E_k + \alpha_{k-1} f(x^*) + (1 - \alpha_{k-1}) f(x_{k-1})\\
		& \hspace*{1cm}- \frac{L\alpha_{k-1}^2}{2}\|v_k - x^*\|^2\\
		& \hspace*{1cm}+ \frac{(L\alpha_{k-1} - \mu)\alpha_{k-1}}{2} \|v_{k-1} - x^*\|^2 \; .
\end{align*}
Reordering the terms, we have
\beq
f(x_k) - f(x^*) + \frac{L\alpha_{k-1}^2}{2}\|v_k - x^*\|^2 \leqslant (1 - \alpha_{k-1})\left(f(x_{k-1}) - f(x^*)\right) + \frac{(L\alpha_{k-1} - \mu)\alpha_{k-1}}{2} \|v_{k-1} - x^*\|^2 + E_k \; .
\label{eq:rate_nesterov_strongly_convex}
\eeq
We can rewrite Eq.~(\ref{eq:rate_nesterov_strongly_convex}) as
\begin{align*}
f(x_k) - f(x^*) + \frac{L\alpha_{k-1}^2}{2}\|v_k - x^*\|^2 \leqslant (1 - \alpha_{k-1})	&\left(f(x_{k-1}) - f(x^*) + \frac{L\alpha_{k-2}^2}{2}\|v_{k-1} - x^*\|^2\right)\\
		&+ \frac{L\alpha_{k-1}^2 - \mu\alpha_{k-1} - (1 - \alpha_{k-1})L\alpha_{k-2}^2}{2}\|v_{k-1} - x^*\|^2\\
		& + E_k \; .
\label{eq:rate_nesterov_strongly_convex_fixed_alpha}
\end{align*}
Using
\[
\alpha_k = \sqrt{\frac{\mu}{L}}
\]
and denoting
\begin{align}
\delta_k & = f(x_k) - f(x^*) + \frac{\mu}{2}\|v_k - x^*\|^2 \; ,
\label{eq:delta_k_def}
\end{align}
we get the following recursion:
\begin{align}
\delta_k & \leqslant \left(1 - \sqrt{\frac{\mu}{L}}\right)\delta_{k-1} + E_k \; .
\end{align}
Applying this relationship recursively, we get
\begin{align}
\delta_k & \leqslant \left(1 - \sqrt{\frac{\mu}{L}}\right)^k \delta_0 + \sum_{t = 1}^{k} E_t\left(1 - \sqrt{\frac{\mu}{L}}\right)^{k-t} \; .
\label{eq:delta_bound}
\end{align}
Since $E_k = \varepsilon_k + \alpha_{k-1}\left\langle e_k+Lf_k,x^* - v_k\right\rangle$, we can bound it by
\begin{align}
E_k &\leqslant \varepsilon_k + \sqrt{\frac{\mu}{L}} \left(\|e_k\| + \sqrt{2L\varepsilon_k}\right)\|v_k - x^*\|
\label{eq:E_k}
\end{align}
using $\|f_k\| \leqslant \sqrt{\frac{2\varepsilon_k}{L}}$.
Plugging Eq.~(\ref{eq:E_k}) into Eq.~(\ref{eq:delta_bound}), we get
\begin{align}
\delta_k & \leqslant \left(1 - \sqrt{\frac{\mu}{L}}\right)^k \left(\delta_0 + \sum_{t = 1}^{k} \left(\varepsilon_t + \sqrt{\frac{\mu}{L}} \left(\|e_t\| + \sqrt{2L\varepsilon_t}\right)\|v_t - x^*\|\right)\left(1 - \sqrt{\frac{\mu}{L}}\right)^{-t}\right) \; .
\label{eq:full_delta_bound}
\end{align}

Again, we shall use Eq.~(\ref{eq:delta_bound}) to first bound the values of $\|v_i - x^*\|$, then the function values themselves.

\subsubsection{Bounding $\|v_i - x^*\|$}

We will now use Lemma~\ref{lemma:bound} to bound the value of $\|v_k - x^*\|$. Since $\|v_k - x^*\|^2$ is bounded by $\frac{2\delta_k}{\mu}$ (using Eq.~(\ref{eq:delta_k_def})), we can use Eq.~(\ref{eq:delta_bound}) to get
\begin{align*}
\|v_k - x^*\|^2 &\leqslant \frac{2}{\mu} \left(1 - \sqrt{\frac{\mu}{L}}\right)^k \left[\delta_0 + \sum_{t=1}^k \varepsilon_t \left(1 - \sqrt{\frac{\mu}{L}}\right)^{-t}\right]\\
				&\hspace*{1cm} + \frac{2}{\sqrt{L\mu}} \sum_{t = 1}^k \left(\|e_t\| + \sqrt{2L\varepsilon_t}\right)\left(1 - \sqrt{\frac{\mu}{L}}\right)^{k-t}\|v_t - x^*\| \; .
\end{align*}
Multiplying both sides by $\left(1 - \sqrt{\frac{\mu}{L}}\right)^{-k}$ yields
\begin{align}
\left(1 - \sqrt{\frac{\mu}{L}}\right)^{-k}\|v_k - x^*\|^2 &\leqslant \frac{2}{\mu} \left[\delta_0 + \sum_{t=1}^k \varepsilon_t \left(1 - \sqrt{\frac{\mu}{L}}\right)^{-t}\right]\nonumber\\
				&\hspace*{1cm} + \frac{2}{\sqrt{L\mu}} \sum_{t = 1}^k \left(\|e_t\| + \sqrt{2L\varepsilon_t}\right)\left(1 - \sqrt{\frac{\mu}{L}}\right)^{-t/2}\left(1 - \sqrt{\frac{\mu}{L}}\right)^{-t/2}\|v_t - x^*\| \; .
\end{align}
Using Lemma~\ref{lemma:bound} with $\displaystyle S_k = \frac{2}{\mu} \left[\delta_0 + \sum_{t=1}^k \varepsilon_t \left(1 - \sqrt{\frac{\mu}{L}}\right)^{-t}\right]$ and $\lambda_i = \frac{2}{\sqrt{L\mu}} \left(\|e_i\| + \sqrt{2L\varepsilon_i}\right)\left(1 - \sqrt{\frac{\mu}{L}}\right)^{-i/2}$, we get
\begin{align}
\left(1 - \sqrt{\frac{\mu}{L}}\right)^{-k/2}\|v_k - x^*\| & \leqslant \frac{1}{2} \sum_{t=1}^k \lambda_t + \left(
\frac{2}{\mu}\delta_0 + \frac{2}{\mu}\widehat{B}_k + \left( \frac{1}{2} \sum_{t=1}^k \lambda_t \right)^2
\right)^{1/2} \; .
\end{align}
with
\begin{align*}
\widehat{B}_k &= \sum_{t=1}^k \varepsilon_t \left(1 - \sqrt{\frac{\mu}{L}}\right)^{-t}
\end{align*}
Since $\widehat{B}_t$ is an increasing sequence and the $\lambda_t$ are positive, we have for $i \leqslant k$
\begin{align*}
\left(1 - \sqrt{\frac{\mu}{L}}\right)^{-i/2}\|v_i - x^*\| &\leqslant \frac{1}{2} \sum_{t=1}^i  \lambda_t + \left(\frac{2}{\mu}\delta_0 +  \frac{2}{\mu}\widehat{B}_i + \left( \frac{1}{2} \sum_{t=1}^i  \lambda_t \right)^2\right)^{1/2}\\
			& \leqslant \frac{1}{2} \sum_{t=1}^i  \lambda_t + \sqrt{\frac{2\delta_0}{\mu}} +\sqrt{\frac{2\widehat{B}_i}{\mu}} + \frac{1}{2} \sum_{t=1}^i  \lambda_t\\
			& = \sum_{t=1}^k  \lambda_t + \sqrt{\frac{2\delta_0}{\mu}} +\sqrt{\frac{2\widehat{B}_k}{\mu}}\; .
\end{align*}
Hence,
\begin{align}
\|v_i - x^*\| &\leqslant \left(1 - \sqrt{\frac{\mu}{L}}\right)^{i/2}\left(\sum_{t=1}^k  \lambda_t + \sqrt{\frac{2\delta_0}{\mu}} +\sqrt{\frac{2\widehat{B}_k}{\mu}}\right)\; .
\label{eq:v_i_bound}
\end{align}

\subsubsection{Bounding the function values}

Denoting $\widehat{A}_k = \frac{\mu}{2}\sum_{t=1}^k  \lambda_t = \sqrt{\frac{\mu}{L}} \sum_t\left(\|e_t\| + \sqrt{2L\varepsilon_t}\right)\left(1 - \sqrt{\frac{\mu}{L}}\right)^{-t/2}$, we obtain after plugging Eq.~(\ref{eq:v_i_bound}) into Eq.~(\ref{eq:full_delta_bound})
\begin{align*}
\delta_k & \leqslant  \left(1 - \sqrt{\frac{\mu}{L}}\right)^k\left[\delta_0 + \sum_{t=1}^k \varepsilon_t \left(1 - \sqrt{\frac{\mu}{L}}\right)^{-t}\right]\\
				&\hspace*{1cm} + \sqrt{\frac{\mu}{L}} \sum_{t = 1}^k \left(\|e_t\| + \sqrt{2L\varepsilon_t}\right)\left(1 - \sqrt{\frac{\mu}{L}}\right)^{k-t/2}\left(\frac{2\widehat{A}_k}{\mu} +\sqrt{\frac{2\delta_0}{\mu}} +\sqrt{\frac{2\widehat{B}_k}{\mu}}\right)\\
	& = \left(1 - \sqrt{\frac{\mu}{L}}\right)^k \left(\delta_0 + \widehat{B}_k + \frac{\mu}{2}\sum_{t=1}^k \lambda_t \left(\frac{2\widehat{A}_k}{\mu} +\sqrt{\frac{2\delta_0}{\mu}} +\sqrt{\frac{2\widehat{B}_k}{\mu}}\right)\right)\\
& = \left(1 - \sqrt{\frac{\mu}{L}}\right)^k \left(\delta_0 + \widehat{B}_k + \widehat{A}_k \left(\frac{2\widehat{A}_k}{\mu} +\sqrt{\frac{2\delta_0}{\mu}} +\sqrt{\frac{2\widehat{B}_k}{\mu}}\right)\right)\\
		& \leqslant \left(1 - \sqrt{\frac{\mu}{L}}\right)^k \left(\sqrt{\delta_0} + \widehat{A}_k\sqrt{\frac{2}{\mu}} + \sqrt{\widehat{B}_k}\right)^2 \; .
\end{align*}
Using the $L$-Lipschitz gradient of $f$ and the fact that $v_0 = x_0$, we have
\begin{align*}
\sqrt{\delta_0} 	& = \sqrt{f(x_0) - f(x^*) + \frac{\mu}{2}\|v_0 - x^*\|^2}\\
				& \leqslant \sqrt{2(f(x_0) - f(x^*))}\; .
\end{align*}
Hence, discarding the term $\|v_k - x^*\|^2$ of $\delta_k$, we have
\begin{align}	
	f(x_k) - f(x^*) & \leqslant \left(1 - \sqrt{\frac{\mu}{L}}\right)^k \left(\sqrt{2(f(x_0) - f(x^*))} + \widehat{A}_k\sqrt{\frac{2}{\mu}} + \sqrt{\widehat{B}_k}\right)^2 \; .
\end{align}
\end{proof}

\small
\bibliographystyle{unsrt}
\bibliography{inexact_prox}

\end{document}